\documentclass{article}

\usepackage{microtype}
\usepackage{graphicx}
\usepackage{subfigure}
\usepackage{booktabs} % for professional tables

\usepackage{graphicx} % more modern
\usepackage{float}	  % for tables
\usepackage{calc} %
\usepackage{caption}
\usepackage{natbib} %citations
\usepackage{acro} %acronyms

\usepackage{amsfonts, amsmath, amsthm} 
\usepackage{IEEEtrantools}
\theoremstyle{definition}
\newtheorem{theorem}{Theorem} 
\newtheorem{definition}{Definition}
\newtheorem{condition}{Condition}

\usepackage{xcolor}
\usepackage{tikz} 
\usetikzlibrary{arrows,shapes}

\usepackage{eurosym}

\usepackage{xspace}

\usepackage{hyperref}

\usepackage[accepted]{icml2018_arxiv} %use icml2018_arxiv to add 'to appear in'

\usepackage{algorithm}
\usepackage{algorithmic}
\definecolor{darkgreen}{rgb}{0,0.3,0}

\def\*#1{\boldsymbol{#1}} %\def\*#1{#1}

\newcommand{\acro}[1]{{\small\textsc{#1}}\xspace}

\newcommand{\NOX}{\acro{NOX}}
\newcommand{\AR}{\acro{AR}}

\newcommand{\MSE}{\acro{MSE}}
\newcommand{\NLL}{\acro{NLL}}
\newcommand{\SIC}{\acro{SIC}}
\newcommand{\RLD}{\acro{RLD}}
\newcommand{\SGV}{\acro{SGV}}

\newcommand{\AM}{\acro{AM}}
\newcommand{\FL}{\acro{FL}}

\newcommand{\GHz}{\acro{GHz}}
\newcommand{\iseven}{{i7}}
\newcommand{\GB}{\acro{GB}}

\newcommand{\RAM}{\acro{RAM}}  
\newcommand{\EPSRC}{\acro{EPSRC}}

\newcommand{\CP}{\acro{CP}}
\newcommand{\CPs}{\acro{CPs}}
\newcommand{\GP}{\acro{GP}}
\newcommand{\GPs}{\acro{GPs}}
\newcommand{\BOCPD}{\acro{BOCPD}}
\newcommand{\BOCPDMS}{\acro{BOCPDMS}}
\newcommand{\BVAR}{\acro{BVAR}}
\newcommand{\BVARs}{\acro{BVARs}}
\newcommand{\SSBVAR}{\acro{SSBVAR}}
\newcommand{\SSBVARs}{\acro{SSBVARs}}
\newcommand{\BAR}{\acro{BAR}}
\newcommand{\BARs}{\acro{BARs}}
\newcommand{\VAR}{\acro{VAR}}
\newcommand{\VARs}{\acro{VARs}}

\newcommand{\BFs}{\acro{BFs}}
\newcommand{\MAP}{\acro{MAP}}
\newcommand{\VARMA}{\acro{VARMA}}
\newcommand{\GARCH}{\acro{GARCH}}

\newcommand{\HMMs}{\acro{HMMs}}
\newcommand{\YWE}{\acro{YWE}}
\newcommand{\OLS}{\acro{OLS}}
\newcommand{\ARGPCP}{\acro{ARGPCP}}
\newcommand{\GPTSCP}{\acro{GPTSCP}}
\newcommand{\NSGP}{\acro{NSGP}}
\newcommand{\ML}{\acro{ML}}
\newcommand{\SRC}{\acro{SRC}}
\newcommand{\QR}{\acro{QR}}

\icmltitlerunning{Spatio-temporal Bayesian On-line Changepoint Detection with Model Selection}

\begin{document} 

\twocolumn[
\icmltitle{Spatio-temporal Bayesian On-line Changepoint Detection with Model Selection %\\ 
           %International Conference on Machine Learning (ICML 2018)}
           }

% It is OKAY to include author informatioBubble Sortn, even for blind


% you can specify symbols, otherwise they are numbered in order
\icmlsetsymbol{equal}{*}

\begin{icmlauthorlist}
\icmlauthor{Jeremias Knoblauch}{stats}
\icmlauthor{Theodoros Damoulas}{stats,cs,ati}
\end{icmlauthorlist}

\icmlaffiliation{stats}{Department of Statistics, University of Warwick,  UK}
\icmlaffiliation{cs}{Department of Computer Science, University of Warwick, UK}
\icmlaffiliation{ati}{The Alan Turing Institute for Data Science \& AI, UK}

\icmlcorrespondingauthor{Jeremias Knoblauch}{j.knoblauch@warwick.ac.uk}
%\icmlcorrespondingauthor{Eee Pppp}{ep@eden.co.uk}

%lsetsymbol{equal}{*}

%\icmlcorrespondingauthor{Cieua Vvvvv}{c.vvvvv@googol.com}
%\icmlcorrespondingauthor{Eee Pppp}{ep@eden.co.uk}

\icmlkeywords{Bayesian Online Changepoint Detection, Spatial Statistics, Time Series, Gaussian Graphical Models, Model Selection}

\vskip 0.3in
]



\printAffiliationsAndNotice{} 

\begin{abstract} 

%The well-known Bayesian On-line Changepoint Detection (BOCD) algorithm is extended from the time series to the space-time setting.
%To this end, 
Bayesian On-line Changepoint Detection is extended to on-line model selection and non-stationary spatio-temporal processes. 
%well-developed - what does it mean? on-line models? 
%The framework is compatible with a variety of well-developed models for temporal data. 
We propose spatially structured Vector Autoregressions (\VARs) for modelling the process between changepoints (\CPs)
and give an upper bound on the approximation error of such models. 
%Make this stronger, i.e. more active formulation (we do this!)
%This is inspired by a recently derived 
%To demonstrate the flexibility of this model family, we give
 %an upper bound on the approximation error of VARs.
 %when approximating processes using VARs. 
%recent developments in representation theory 
The resulting algorithm performs prediction, model selection and \CP detection on-line. Its time complexity is linear and its space complexity constant, and thus it is two orders of magnitudes faster than its closest competitor. 
%Be stronger. comparable -> better or something
In addition, it outperforms the state of the art for multivariate data.

\end{abstract} 

\section{Introduction}
\label{Section_introduction}

%other literature
%mention separability here!?
Real-world spatio-temporal processes are often poorly modelled by standard inference methods that assume stationarity in time and space.
A variety of techniques have been developed for modelling non-stationarity in time via changepoints (\CPs), ranging from methods for Gaussian Processes (\GPs)  \citep{Osborne},  the Lasso \citep{lin2017sharp} or the Ising model \citep{fazayeli2016generalized} over approaches using density ratio estimation \citep{liu2013change} and kernel-based methods exploiting M-statistics \citep{li2015m} to framing \CP detection as time series clustering \citep{khaleghi2014asymptotically}. In contrast, %with the notable exception of
 \CP inference allowing for non-stationarity in space \cite{Flaxman} has received comparatively little attention.
% WRITE ON OSBORNE'S GP PAPER COMPARISON
\begin{figure}[t!]
\vskip 0.1in
\begin{center}
\centerline{\includegraphics[trim= {0.5cm 0.7cm 1.4cm 1.225cm}, clip, %{1.5cm 3cm 2.2cm 3.6cm},clip, 
width=1.00\columnwidth]{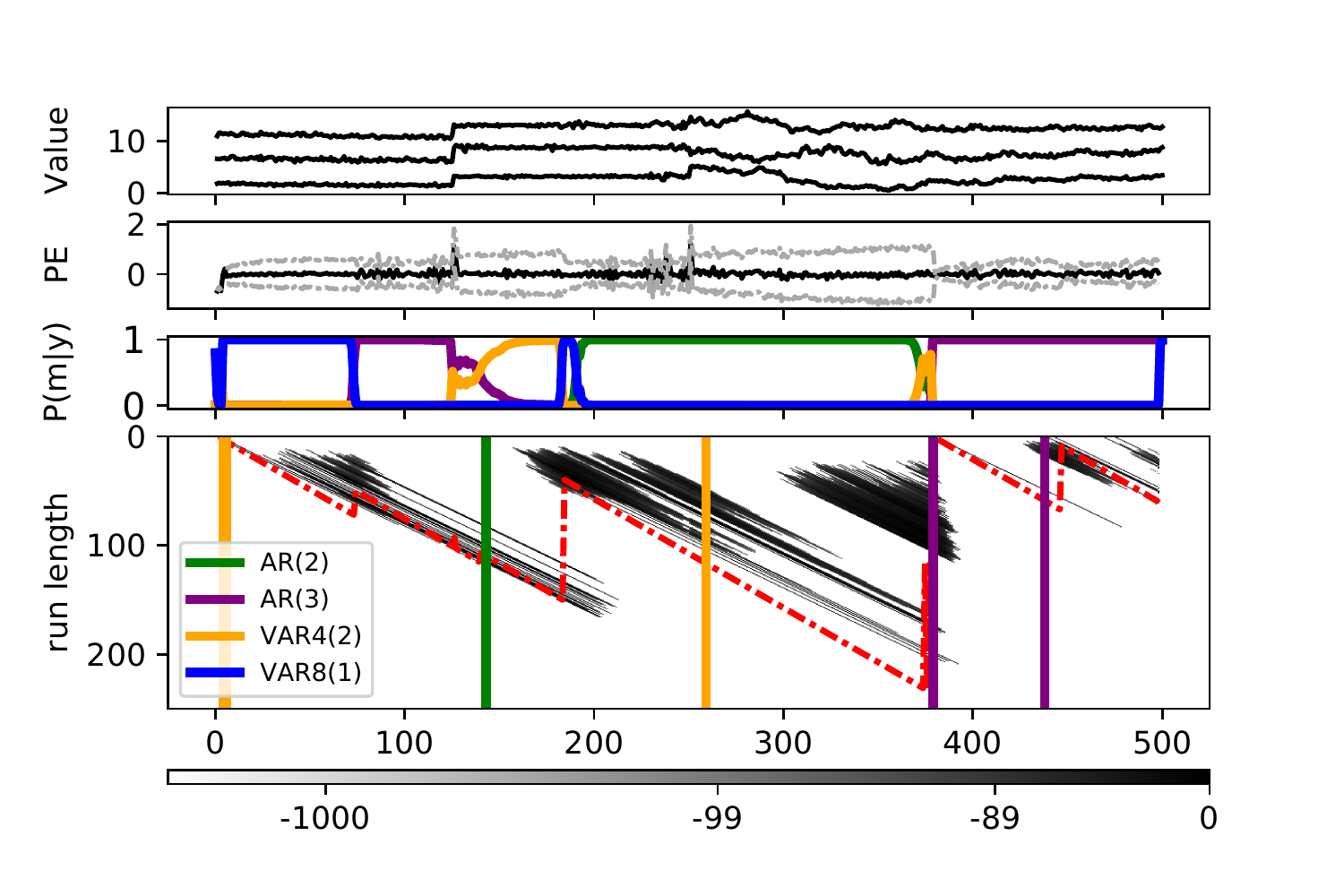}}
%left lower right upper
%DO: Check if we can make textbf a little less fat
\caption{\textit{Bayesian On-line Changepoint Detection with Model Selection (\BOCPDMS)}: {\color{darkgray}\textbf{Panel 1:}} Artificial data across times ${1-500}$ for a regular spatial grid with $4$- and $8$-neighbourhood dependency structure as in Fig. \ref{SSBVAR_graph}, 
%where Model universe $\mathcal{M}$ uses AR and Spatially Structured BVAR models with $4$-neighbourhood and lag lengths $1-3$, see Fig \ref{SSBVAR_graph}.  
%Nonstationary data of different grid locations across times ${1-500}$.
 {\color{darkgray}\textbf{Panel 2:}} prediction error (black) and variance (gray). {\color{darkgray}\textbf{Panel 3:}} Model posteriors $p(m_t|\*y_{1:t})$. {\color{darkgray}\textbf{Panel 4:}} log run-length distribution (grayscale), its maximum (red) and \MAP segmentation of \CPs and models in corresponding colors.}
\label{picture_demo}
\end{center}
\vskip -0.2in
\end{figure}
%our research
%non-bayesian techniques doing this?

We offer the first on-line solution to this problem by modeling non-stationarity in both space and time. \CPs are used to model non-stationarity in time, and the use of spatially structured Bayesian Vector Autoregressions (\SSBVAR) circumvents the assumption of stationarity in space. We unify %the approaches in 
\citet{BOCD} and \citet{FearnheadOnlineBCD} into an inference procedure for  \textbf{on-line} \textbf{prediction}, \textbf{model selection} and  \textbf{\CP detection}, see Fig. \ref{picture_demo}. 
%This conjunction is possible as both methods build on the Product Partition Model  \cite{PPM}.
{
Our construction exploits that both algorithms use Product Partition Models \cite{PPM}, which assume independence of parameters conditional on the \CPs and independence of observations conditional on these parameters.
}

{
%reltad to Murphy MVTS
Our method can be seen as modified on-line version of \citet{MurphyMVTSBCP}.  In their method, inference is off-line, the model universe $\mathcal{M}$ is built during execution and  multivariate dependencies are %expressed via correlated errors 
restricted to decomposable graph. In contrast, our procedure  specifies $\mathcal{M}$ before execution, but runs on-line and does not restrict dependencies.
%also related to GP paper
The closest competing on-line procedure in the literature thus far is the work of \citet{GPBOCD}, which develops   Gaussian Process (\GP) \CP models for Bayesian On-line Changepoint Detection (\BOCPD). Though our results suggest that parametric models may be preferable to \GP models, the latter can still be integrated into our method  as elements of the model universe $\mathcal{M}$ without any further modifications.
}

{
In summary, we make three contributions:
Firstly, we substantially augment the existing work on \BOCPD by allowing for model uncertainty. Unlike previous extensions of the algorithm \citep[e.g.][]{BOCD, GPBOCD}, this avoids having to guess a single best model a priori.
Secondly, we introduce \SSBVARs as the first class of models for multivariate inference within \BOCPD. 
Thirdly, we demonstrate that %computational + performance improvement
using a collection of parametric models %within \BOCPD 
can outperform nonparametric \GP models in terms of prediction, \CP detection and computational efficiency.
 }
The structure of this paper is as follows: Section \ref{Algorithm}  generalizes the \BOCPD algorithm of \citet{BOCD}, henceforth \AM, by integrating it with the approach of \citet{FearnheadOnlineBCD}, henceforth \FL. In so doing, we arrive at \BOCPD with Model Selection, henceforth \BOCPDMS. Section \ref{VAR} proposes \VAR models for non-stationary processes within the \BOCPD framework. 
%are excellent approximations to processes that are stationary in time, but not necessarily in space. 
This motivates populating the model universe $\mathcal{M}$ with spatially structured \BVAR (SSBVAR) models. Sections \ref{hyperpar_opt}--\ref{computational_complexity} address computational aspects.
Section \ref{results} demonstrates the algorithm's advantages on real world data.

\section{BOCPDMS}\label{Algorithm}

%ADD $Y_t$'s dimension either here or in $\mathcal{Y}$

%USE ABBREVIATIONS FOR ADAMS AND FEARNHEAD ', henceforth AM'

Let $\{\*Y_t\}_{t=1}^{\infty}$ be a data stream with an unknown number of CPs.
% at times $\{\tau_i\}$ so that $\tau_i < \tau_{i+1}$. 
%It is in general  difficult to obtain a distribution over number and locations of these CPs. 
Focusing on univariate data, \FL % \citet{FearnheadOnlineBCD} 
and \AM %\citet{BOCD} 
tackled inference by tracking the posterior distribution for the most recent \CP. 
While \FL allow the data to be described by different models between \CPs, \AM %\citeauthor{BOCD} 
only allow for a single model. However, \AM perform one-step-ahead predictions, whereas \FL %\citeauthor{FearnheadOnlineBCD} 
do not. Instead, they propose a Maximum A Posteriori (\MAP) segmentation for \CPs and models. 
In the remainder of this section, we unify both inference approaches. We call the resulting algorithm \BOCPD with model selection (\BOCPDMS), as it performs prediction, \MAP segmentation and model selection on-line.
%recursions presented in both papers. In doing so,  prediction and MAP estimation become outputs of the same algorithm.
%While \citet{FearnheadOnlineBCD} solved this issue using a state-space representation, \citet{BOCD} ingeniously circumvented the issue by defining the run-length $r_t$ at time $t$ as the number of time periods that have passed since the last CP. However, unlike the algorithm in \citet{FearnheadOnlineBCD}, they did not allow for a collection of models $\mathcal{M}$ to capture the data between change points. In the remainder of this section, we recast the state-space based inference of \citet{FearnheadOnlineBCD} in terms of the run-length framework. 

\subsection{Run-length \& model universe}\label{run-length-model-univ-subsec}

The \textit{run-length}  $r_t$ at time $t$ is %defined as 
%$r_t = t - \tau^{\ast}(t)$, where $\tau^{\ast}(t)$ is 
defined as the time since the most recent \CP at time $t$, so $r_t = 0$ corresponds to a \CP at time $t$.
%Though trivial, to the best of our knowledge this has not been done before. %Efficiently updating the run-length posterior $\mathbb{P}(r_t|Y_{1:(t-1)})$ to 
%To give an example, suppose the process $\{Y_t\}$ has a only two CPs, one at time $0$ and one at time $\tau$. In this case, the run-length of $\{Y_t\}$ is given by $r_t = \mathbbm{1}_{t<\tau}\cdot (t-1) + \mathbbm{1}_{t \geq \tau} \cdot(t-\tau)$.
%
%We note here that the same idea can be formulated in terms of a state space sequence as is done in \citet{FearnheadOnlineBCD}, but this is somewhat more cumbersome and less intuitive.
%
%PLOT MVTS WITH CPs AT 0, Tau HERE
Suppose that data between successive \CPs can be described by Bayesian  models collected in the \textit{model universe} $\mathcal{M}$. 
%Each model $M$ is defined over the space $\mathcal{Y}$ that the realizations of the process $\{Y_t\}$ live on.
%
%CHANGE FORMULATION 
%
For the process $\{\*Y_t\}$ on $\mathbb{R}^S$, a model $m \in \mathcal{M}$ with finite memory of length $L\in\mathbb{N}_0$ consists of an observation density $f_m(\*Y_t=\*y_t|\*\theta_m, \*y_{(t-L):(t-1)})$ on $\mathbb{R}^S$ and a parameter prior $\pi_m(\*\theta_m)$ on $\*\Theta_m$ depending on hyperparameters $\*\nu_m$. The notion of $\mathcal{M}$ is due to \FL % \citet{FearnheadOnlineBCD} 
%has so far received relatively little attention in the Machine Learning community: Both in \citet{BOCD} and \citet{GPBOCD}, the algorithms are defined for a {single} model.This is somewhat puzzling, since introducing the model universe 
and allows for model uncertainty amongst models developed for \BOCPD. For instance, $m \in \mathcal{M}$ could be a \GP \cite{GPBOCD}, a time-deterministic regression \cite{FearnheadSignalProcessing} or a mixture distribution \cite{CaronDoucet}. 
%The resulting Algorithm \ref{Algorithm_BOCPDMS} can incorporate these different models and then not only perform prediction, but also {on-line} model selection. 
%The remainder of the section derives the necessary recursions and summarizes the results in Algorithm .
%In section \ref{VAR}, we introduce a class of VAR-based models that can exploit this model selection feature to capture spatio-temporal dependency while remaining computationally efficient. %While the formulation of the algorithm remains valid for this large variety of model classes, we propose a new class of models specifically designed to be computationally efficient while capturing spatio-temporal dependency. 

\begin{algorithm}[t!]
%\floatname{algorithm}{Algoritmo}
   \caption*{\textbf{\BOCPD with Model Selection (\BOCPDMS)}}
   \label{Algorithm_BOCPDMS}
\begin{algorithmic}
   \STATE {\bfseries Input at time $0$:} model universe $\mathcal{M}$; hazard $H$; prior $q$
   \STATE {\bfseries Input at time $t$:} next observation $\*y_t$ 
   \STATE {\bfseries Output at time} $t$: $\widehat{\*y}_{(t+1):(t+h_{\max})}$, $S_t$, $p(m_t|\*y_{1:t})$
    \\[5pt]
   \FOR{ next observation $\*y_t$ at time $t$}   \vspace*{0.15cm}
   \STATE{$//$ STEP I:  Compute model-specific quantities} \vspace*{0.05cm}
   	\FOR{ $m \in \mathcal{M}$ }
   		\IF{ $t-1 = \text{lag\_length(m)}$ }
   			\STATE [I.A] Initialize $p(\*y_{1:t}, r_t=0, m_t=m)$ with prior
   		\ELSIF{ $t-1 > \text{lag\_length(m)}$ }
   			\STATE [I.B.1] Update $p(\*y_{1:t}, r_t, m_t=m)$  via \eqref{growth_probability}, \eqref{cp_probability}
   			\STATE [I.B.2] Prune model-specific run-length distribution
   			\STATE [I.B.3] Perform hyperparameter inference via \eqref{gradient_descent_caron}
   		\ENDIF 
   	\ENDFOR \vspace*{0.15cm}
   	%\hrulefill
   	\STATE {$//$ STEP II:  Aggregate over models} \vspace*{0.05cm}
   	\IF{ $t >= \min(\text{lag\_length(m)})$ }
   		\STATE [II.1] Obtain joint distribution over $\mathcal{M}$ via \eqref{evidence}--\eqref{conditional_model_posterior}
   		\STATE [II.2] Compute  	
   			\eqref{posterior_predictive}--\eqref{MAP_estimator}  
   		\STATE [II.3] \textbf{Output:} $\widehat{\*y}_{(t+1):(t+h_{\max})}, S_t, p(m_t|\*y_{1:t})$  
   	\ENDIF 
   \ENDFOR
\end{algorithmic}
\end{algorithm}

\subsection{Probabilistic formulation \& recursions}

%Allowing for mutliple models as in \citet{FearnheadOnlineBCD} but tracking the joint probabilities for prediction as in \citet{BOCD}, we unify both inference frameworks. 
Denote by $m_t$ the model describing $\*y_{(t-r_t):t}$, i.e. the data since the last \CP. Given hazard function $H:\mathbb{N} \to [0,1]$, and  model prior  $q:\mathcal{M} \to [0,1]$, the prior beliefs are
%we have the following prior beliefs:
%I need to have curly brackets here.
\begin{IEEEeqnarray}{l}
\IEEEyesnumber \IEEEyessubnumber*
	p(r_t|r_{t-1})  =    \begin{cases}
										1 - H(r_{t-1} + 1)	& \text{ if } r_t=r_{t-1}+1  \\
										H(r_{t-1} + 1)  	& \text{ if } r_t=0 \\
										0					& \text{ otherwise.}
									\end{cases}
		\label{run_length_prior}\quad\quad\\ %maybe easier to write P(r_t=r)=H(r, \lambda) and then explain conditional
q(m_t|m_{t-1},r_t)  =  \begin{cases}
											\*{1}_{m_{t-1}}(m_t)	& \text{ if } r_t=r_{t-1}+1 \\
											 q(m_t)  	& \text{ if } r_t=0. \\
									\end{cases} \label{model_prior}
\end{IEEEeqnarray}
%Because it is more convenient to work with in the recursive formulation, equation \eqref{run_length_prior} specifies the prior belief on the run-length conditionally. 
Eq. \eqref{model_prior} implies that the model at time $t$ will be equal to the model at time $t-1$ unless a \CP occured at $t$, in which case the next model $m_t$ will be a random draw from $q$. 
At time $t$, the algorithm requires for all possible models $m$ and run-lengths $r_t$ the computation of the  posterior predictives
\begin{IEEEeqnarray}{lll}
	 & & f_m(\*y_{t}| \*y_{1:(t-1)}, r_t) \nonumber \\ %= %& & \nonumber \\
	 %\hspace*{-0.15cm} 
	 & = & \int_{\Theta_{m}} 
	%\hspace*{-0.1cm}
	f_m(\*y_{t}|\*\theta_{m})\pi_m(\*\theta_{m}|\*y_{(t-L-r_t):(t-1)}) d\*\theta_{m}. %\label{Pred_probs}%\hspace*{0.2cm}
\end{IEEEeqnarray}
To make the evaluation of this integral efficient, one can use conjugate models \cite{MurphyMVTSBCP} or approximations \cite{TurnerVB, CHAMP}, which make the following recursion efficient, too:
\begin{IEEEeqnarray}{rCl}
\IEEEyesnumber
	& & p(\*y_{1:t}, r_{t}, m_{t}) =  \nonumber \\
	& & \sum_{m_{t-1}}\sum_{r_{t-1}}\Bigl\{ 
		f_{m_t}(\*y_t|\*y_{1:(t-1)},r_{t}) 
		q(m_t|\*y_{1:(t-1)},  r_t,m_{t-1}) \nonumber \\
	&& \quad\quad\quad\quad\quad	p(r_t|r_{t-1}) 
		p(\*y_{1:(t-1)}, r_{t-1},m_{t-1})\Bigr\}. \label{recursion}
\end{IEEEeqnarray}
The recursion in  \AM % \citet{BOCD} 
is the special case for $|\mathcal{M}| = 1$.
For $|\mathcal{M}| > 1$, $q(m_t|m_{t-1}, r_t, \*y_{1:(t-1)})$ arises as a new term, which for $\*{1}_{a}$ as the indicator function of $a$ is given by
%MULTIPLE EQUATIONS WITH A SINGLE EQUATION
\begin{IEEEeqnarray}{C}
	%\mathbb{P}(m_t| y_{1:(t-1)}, r_{t},m_{t-1})  \nonumber \\
									%=	 
									\hspace*{-0.4cm}\begin{cases}
													\*{1}_{m_{t-1}}(m_t)q(m_{t-1}|\*y_{1:(t-1)}, r_{t-1})	& \text{ if } r_t=r_{t-1}+1 \\
												 q(m_t)  	& \text{ if } r_t=0.
												 %0			& \text{ otherwise}.
										\end{cases}\ \label{model_posterior_probability}
\end{IEEEeqnarray}
%Combining \eqref{run_length_prior}--\eqref{model_posterior_probability}, 
Next,
 define the \textit{growth-} and \textit{changepoint probabilities} as
\begin{IEEEeqnarray}{lll}
\IEEEyesnumber \IEEEyessubnumber*
 	&&  \hspace*{-0.4cm} p(\*y_{1:t}, r_{t}=r_{t-1}+1, m_{t})  =  \nonumber \\
 	&  & 	f_{m_t}(\*y_t| \*y_{1:(t-1)},r_t) 
 			p(\*y_{1:(t-1)}, r_{t-1}, m_{t-1}) \times\label{growth_probability}  \\
 	&&	(1-H(r_t))	q(m_{t-1}|\*y_{1:(t-1)}, r_t),\nonumber \\
%\end{IEEEeqnarray}
%and the \textit{changepoint probabilities} 
%\begin{IEEEeqnarray}{lll}
%\IEEEyessubnumber*
 	&& \hspace*{-0.4cm} p(\*y_{1:t}, r_{t}=0, m_{t})  =\nonumber \\  
 	& & f_{m_t}(\*y_t| \*y_{1:(t-1)}, r_{t})q(m_t) \times \label{cp_probability} \\
 	& & 		\sum_{m_{t-1}}\sum_{r_{t-1}}\Bigl\{ 
 	 		H(r_{t-1} + 1) p(\*y_{1:(t-1)}, r_{t-1}, m_{t-1}) \nonumber
 			\Bigr\}.
\end{IEEEeqnarray}
%Again, note that if conditions on $m_t$ to get $d\mathbb{P}(y_{1:t}, r_{t}| m_{t}) = d\mathbb{P}(y_{1:t}, r_{t}, m_{t})/q(m_t)$,  one arrives at the recursions in \citet{BOCD}.
The evidence can then be calculated via Eq. \eqref{evidence}, which in turn allows calculating the joint model-and-run-length distribution  \eqref{model_and_run_length}, the model posterior \eqref{model_posterior}, as well as the model-specific \eqref{model_specific_run_length}  and global \eqref{global_run_length} run-length distributions:
\begin{IEEEeqnarray}{rCl}
 \IEEEyesnumber \IEEEyessubnumber*
 %	d\mathbb{P}(y_{1:t}|m_t) & = & \sum_{r_{t}} d\mathbb{P}(y_{1:t}, r_t, m_t) \\
	p(\*y_{1:t}) & = & \textstyle \sum_{m_t}\sum_{r_t} p(\*y_{1:t}, m_t, r_t) \label{evidence} \\
	p(r_t,m_t| \*y_{1:t}) & = & p(\*y_{1:t}, r_t, m_t)/p(\*y_{1:t}) \label{model_and_run_length} \\
	p(m_t|\*y_{1:t}) & = & \textstyle\sum_{r_t}p(r_t,m_t| \*y_{1:t}) \label{model_posterior} \\
	p(r_t|m_t, \*y_{1:t}) & = & p(r_t, m_t|\*y_{1:t})/p(m_t|\*y_{1:t}) \label{model_specific_run_length} \\
	p(r_t|\*y_{1:t}) & = &\textstyle\sum_{m_t}p(r_t,m_t| \*y_{1:t})  \label{global_run_length} 	\\
%\end{IEEEeqnarray}
%These distributions finally allow computation of %\eqref{model_posterior_probability} via
%\begin{IEEEeqnarray}{rCl}
\hspace*{-0.5cm}q(m_{t-1}|\*y_{1:(t-1)}, r_{t-1}) & = &\frac{p(m_{t-1}, r_{t-1}|\*y_{1:(t-1)})}{p(r_{t-1}| \*y_{1:(t-1)})}. \label{conditional_model_posterior}
\end{IEEEeqnarray}
%MAKE SURE WE SELL OUR DERIVATION, i.e. what are advantages
Eq. \eqref{conditional_model_posterior} is the conditional model posterior from %appearing in the recursion via
 Eq. \eqref{model_posterior_probability}.
%
%SELL MORE THE ADDITIONS TO FL
%Framework change? I.e., by bringing fearnhead into joint land we do more stuff?
%
%In addition to Eq. \eqref{global_run_length}
 Eq. \eqref{global_run_length} is arrived at directly in \FL and used for on-line \MAP segmentation. By framing our derivations in the run-length framework of \AM, we additionally obtain  \eqref{model_posterior_probability}--\eqref{model_specific_run_length}, thus enabling % derivation %has the same computational cost and allows
   on-line prediction and model selection %via  \eqref{model_and_run_length} and on-line model selection via  \eqref{model_posterior_probability} 
 at the same computational cost.
%For $|\mathcal{M}|=1$,  \eqref{growth_probability}--\eqref{evidence} and  \eqref{global_run_length}  are derived in \citet{BOCD}.
%The next section explains how CP detection and prediction are performed on-line via equations \eqref{model_and_run_length}--\eqref{global_run_length}. 

\subsection{On-line algorithm outputs}

\textbf{Prediction:} %$1$-step-ahead prediction as in AM %\citet{BOCD} 
 Recursive $h$-step-ahead forecasting uses  \eqref{model_and_run_length}:%for BOCDPMS is possible via Eq. \eqref{model_and_run_length}:
\begin{IEEEeqnarray}{lll}
%\IEEEyesnumber \IEEEyessubnumber*
&&p(\*Y_{t+h}|\*y_{1:t}) \nonumber \\ 
& = &	\sum_{r_t,m_t}\Bigl\{ p(\*Y_{t+h}|\*y_{1:t},\widehat{\*y}_t^h, r_t, m_t) 
 	 p(r_t,m_t| \*y_{1:t}) \Bigr\}, \quad \label{posterior_predictive}
\end{IEEEeqnarray}
%Here, $ \mathbb{P}(r_t,m_t| y_{1:t})$ serves as a weight for each conditional posterior predictive $d\mathbb{P}(Y_{t+h}|y_{1:t}, r_t, m_t)$. 
where $\widehat{\*y}_t^h = \emptyset$ if $h=1$ and $\widehat{\*y}_t^h = \widehat{\*y}_{(t+1):(t+h-1)}$ otherwise, with $\widehat{\*y}_{t+h} = \mathbb{E}(\*Y_{t+h}|\*y_{1:t},\widehat{\*y}_t^h)$ the recursive forecast.
%Unlike for iid-models as proposed by \citet{BOCD}, 
%In contrast with iid-models as in \citet{BOCD}, 
%Such $h$-step-ahead forecasting is attractive for time-dependent models such as BVAR. 
%One way of interpreting this is to say that each tuple $(r_t, m_t)$ consitutes a hypothesis, and that prediction is performed by weighing each hypothesis with the evidence for that hypothesis contained in the data. 
%If one is interested in summary quantities of the posterior predictive instead -- the posterior mean or variance, say -- these can be obtained using these weights in the exact same way.
%RECURSIVE FORECASTING USING VAR MODEL
%While In practice, one rarely cares about \eqref{posterior_predictive} itself, and is more interested in summary quantities like the posterior mean or variance.

\textbf{Tracking the model posterior/Bayes Factors:}
One of the novel capabilites of the algorithm is on-line
monitoring  of the model posterior via Eq. \eqref{model_posterior}. This is attractive when structural changes in the data happen slowly and are not captured well by \CPs. % because the data evolve slowly, and CPs  
In this case, %one can still run the algorithm and focus on the model posterior instead, since 
$\mathbb{P}(m_t|\*y_{1:t})$  can be used to identify periods of change, see Fig. \ref{Temperatures}.
%,  can be monitored. Opposed to CPs, this picks up soft changes, as opposed to. 
%Plotting this distribution through time gives a way of assessing soft change, as opposed to the hard change encoded by CPs. 
%Similarly, the model posterior can be transformed into a set of Bayes Factors at each time point $t$ to quantify how much more evidence there is for $M_1$ when compared with $M_2$ by computing
For pairwise comparisons, Bayes Factors can be monitored, too:
%NOTE: DO WE NEED BAYES FACTOR FOMRULA?
\begin{IEEEeqnarray}{rCl}
%\IEEEyessubnumber*
	\text{BF($m_1$, $m_2$)}_t & = & 
	\dfrac{p(m_t = m_1 |\*y_{1:t})\cdot q(m_2)}{p(m_t = m_2 |\*y_{1:t})\cdot q(m_1)}.\label{BayesFactors}
\end{IEEEeqnarray}

 %If we need space compress first sentence
\textbf{Maximum A Posteriori (\MAP) segmentation:} 
%Framing the problem in terms of the last CP implies that information about earlier CPs is discarded. 
%only has a posterior over the location of the most recent CP and stores no information about the CPs prior to . %If one cares about retrospecive analysis of the CP structure in a data stream, this requires a remedy. 
%Storing the run-length distribution at each time point $t$ instead of overwriting it increases the space-complexity by an order of magnitude. 
%\citet{FearnheadOnlineBCD}
%circumvents discarding information about CPs before the most recent %one. 
For  $\text{MAP}_t$  the density of the \MAP-estimate of models and \CPs before $t$ and $\text{MAP}_0 = 1$, \FL{}'s recursive estimator is given by
\begin{IEEEeqnarray}{rCl}
	%\IEEEyessubnumber*
	\text{MAP}_t = \max_{r,m}\Bigl\{ p(\*y_{1:t}, r_t=r, m_t=m)  \text{MAP}_{t-r-1} \Bigr\}. \quad \; \;\;\;\label{MAP_estimator}
\end{IEEEeqnarray}
For $r^{\ast}_t, m^{\ast}_t$ maximizers for time $t$, the \MAP segmentation is $S_t = S_{t-r^{\ast}_t-1} \cup \{(t-r^{\ast}_t, m^{\ast}_t) \}$, $S_0 = \emptyset$, where $(t',m_{t'}) \in S_t$ means a \CP at $t'\leq t$, with $m_{t'} \in \mathcal{M}$ the model for $\*y_{t':t}$. 
%Clearly, one major drawback of framing the problem such that all information about CPs is contained in the run-length distribution is that it only tracks the very last CP. 

\section{Building a spatio-temporal model universe}\label{VAR}

%The elements $s \in \mathcal{S}$ can correspond to the coordinates of a sensor array measuring air pollution or to individual counties for which one measures some aggregate quantity like criminal activity and average earnings. 
%
%Change 

The last section derived \BOCPDMS for arbitrary data streams $\{\*Y_t\}$. 
Next, we propose models for $\mathcal{M}$ if $\{\*Y_t\}$ can be mapped into a space $\mathbb{S}$. %is multivariate or spatio-temporal.
Let $\mathcal{S}$ with $|\mathcal{S}| = S$ be a set of spatial locations in $\mathbb{S}$ with measurements $\*Y_t = (Y_{t,1}, Y_{t,2}, \dots, Y_{t, S})^T$ recorded 
at times $t=1,2,\dots$
% some of which are CPs. %Further, let $Y_t = (Y_{t,1}, Y_{t,2}, \dots, Y_{t, S})^T$.

%MENTION THAT EVERYTHING APPLIES AS LONG AS Y CAN BE MAPPED INTO SPACE

\subsection{Bayesian \VAR (\BVAR)}

%Suppose that $\mathcal{S}$ is endowed with some set of neighbourhoods $N(\mathcal{S}) = \{ \{N_i(s)\}_{i=1}^n: s \in \mathcal{S}\}$. 
Inference on $\{\*Y_t\}$ can be drawn using conjugate Bayesian Vector Autoregressions (\BVAR) with lag length $L$ and $E$ additional %exogeneous 
variables $\*Z_t$ as elements of model universe $\mathcal{M}$:
\begin{IEEEeqnarray}{rCl}
	\sigma^2 & \sim & \text{InverseGamma}(a,b)\IEEEyesnumber \IEEEyessubnumber* \label{BVAR_eq_1}\\
	\*\varepsilon_t|\sigma^2 & {\sim} & \mathcal{N}(\*0, \sigma^2 \cdot \*\Omega) \label{BVAR_eq_2}\\
	\*c|\sigma^2 & \sim & \mathcal{N}(\*0, \sigma^2 \cdot \*V_c)\label{BVAR_eq_3} \\
	\*Y_t & = & \*\alpha + \*B \*Z_t + \textstyle \sum_{l = 1}^L \*A_l \*Y_{t-l} + \*\varepsilon_t. \label{BVAR_eq_4}
\end{IEEEeqnarray}
Here, $\*A_l, \*B$ are $S\times S$, $S \times E$ matrices, $\*c = (\*\alpha, \text{vec}(\*B), \text{vec}(\*A_1),  \text{vec}(\*A_2), \dots  \text{vec}(\*A_L))^T$ is a vector of $S \cdot (LS + 1 + E)$  model parameters. Scalars $a,b>0$,  matrix $\*V_c$, and diagonal matrix $\*\Omega$ are hyperparameters. 
%We propose amending this by using a sparsity-pattern induced by neighbourhood sequences as defined in section \ref{nbh_sequences_section}.

\subsection{Approximating processes using \VARs}\label{section_approx_var}

Modelling $\{\*Y_t\}$ as \VAR is attractive, as many complex non-linear processes have \VAR representations, including \HMMs, time-stationary \GPs as well as multivariate \GARCH and fractionally integrated \VARMA processes \cite{BaxterIE2,VARFARIMA}. Performance guarantees for \VAR approximations to such processes are derived using  Baxter's Inequalitiy with multivariate versions of results in \citet{HannanKavalieris}.
%We formally state this in a theorem which heavily draws  on the findings in \citet{VARSieve}.
%
%The following section argues that a wide class of spatio-temporal processes that are stationary in time (but not necessarily in space) can be approximated well by finite-order VAR models. This motivates using the BVAR formulation given in the last section as a model for a wide range of spatio-temporal data generating mechanisms satisfying some weak regularity conditions. 

%\textcolor{blue}{DEFINE CONDITION A (in appendix)}

\begin{theorem}\label{VAR_Thm}
Let $\{\*Y_t\}$ be a time-stationary spatio-temporal process with spectral density satisfying regularity condition \textbf{A} in the Appendix, $||\cdot ||$ a matrix norm, $\mathbb{E}(\*Y_t) = 0$, $\mathbb{E}(\*Y_t\*Y_t^T)<\infty$, $\sum_{h=-\infty}^{\infty}(1 + |h|)^3 ||\mathbb{E}[\*Y_{t}\*Y_{t+h}']|| < \infty$. Then (1)--(3) hold. \\[-18pt]
\begin{itemize}
\item[(1)]$\*Y_t = \sum_{i=1}^{\infty}\*A_i\*Y_{t-i} + \*\varepsilon_t$ for matrices $\{\*A_l\}_{l\in\mathbb{N}}$ and  $\mathbb{E}(\varepsilon_t)= 0$,  $\mathbb{E}(\*\varepsilon_t\*\varepsilon_t')= \*D$, $\*D$ diagonal.\\[-18pt]
\item[(2)] For $\*Y_t = \sum_{l=1}^{L}{\*A}^L_l\*Y_{t-l} + e_t$  with $\{{\*A}^L_l\}_{l=1}^L$ the best linear projection coefficients, $\exists L_0:\forall L>L_0$, $\sum_{l=1}^L(1+|l|)^3||\*A^L_l - \*A_l|| \leq C \cdot \sum_{l=L+1}^{\infty}(1+|l|)^3||\*A_l||$ with $C$ constant.\\[-18pt]
\item[(3)] Using $T$ observations with $L = \mathcal{O}([T/\ln(T)]^{1/6})$ to estimate ${\*A}^L_l$ as \MAP  $\widehat{\*A}^L_l$ of \eqref{BVAR_eq_1}--\eqref{BVAR_eq_4},  it holds that $L(T)^2\sum_{l=1}^{L(T)}||\widehat{\*A}^{L(T)}_l - {\*A}^{L(T)}_l|| \overset{P}{\to} 0$ as $T \to \infty$.\\[-18pt]
\end{itemize}
\end{theorem}
\begin{proof}\renewcommand{\qedsymbol}{}
 Part (1) is shown in \citet{VARFARIMA}, part (2) in Lemma 3.1 of \citet{VARSieve}. Part (3) follows by their {Remark 3.3} if we can prove that the \MAP estimator $\hat{\*c}(L(T))$ of $\*c$ %defined via \eqref{BVAR_eq_1}--\eqref{BVAR_eq_4} 
equals its Yule-Walker estimator (\YWE) as $T \to \infty$. Let $\*B=\*0$, $\*\alpha =\*0$ and note that \YWE equals \OLS as $T \to \infty$. With $\*X_{1:T}$ the regressor matrix of $\*Y_{t-L(T):t}$, $\hat{\*c}(L(T)) = (\*X_{1:T}'\*X_{1:T} + \*V_c^{-1})^{-1}(\*X_{1:T}'\*Y_{1:T})$. 
%Define $\hat{c}'(L(T))$ as OLS estimate. 
Then, part (3) holds as %OLS $\to$ MAP, %$\hat{c}(L(T)) \overset{P}{\to} \hat{c}'(L(T))$ 
 %as $T \to \infty$, 
 \OLS  $\overset{P}{\to}\mathbb{E}(\*X_{1:T}'\*X_{1:T})^{-1}\mathbb{E}(\*X_{1:T}'\*Y_{1:T})$ and %as $T \to \infty$ and
 \begin{IEEEeqnarray}{rCl}
 \hat{\*c}(L(T))	& =  & (\*X_{1:T}'\*X_{1:T} + \*V_c^{-1})^{-1}(\*X_{1:T}'\*Y_{1:T}) \nonumber \\
 	& = & ( \frac{1}{T} \*X_{1:T}'\*X_{1:T} + \frac{1}{T} \*V_c^{-1})^{-1} \frac{1}{T}(\*X_{1:T}'\*Y_{1:T}) \nonumber \\
 	& \overset{P}{\to} & \mathbb{E}(\*X_{1:T}'\*X_{1:T})^{-1}\mathbb{E}(\*X_{1:T}'\*Y_{1:T}). \rlap{$\qquad\qquad\: \Box$} \nonumber \\[-35pt] \nonumber
 \end{IEEEeqnarray}
\end{proof}
In Thm. \ref{VAR_Thm}, assuming $\mathbb{E}(\*Y_t) = \*0$  is without loss of generality: If $\mathbb{E}(\*Y_t) = \*\alpha + \*B\*Z_t$, define $\*Y_t^{\ast} = \*Y_t - (\*\alpha + \*B\*Z_t)$ and apply the theorem to $\{\*Y_t^{\ast}\}$. 
Moreover, the results do \textit{not} require stationarity in space. 
Lastly, part {(3)} suggests a principled way of picking lag lengths $\mathcal{L}$ for \BVAR models based on functions $L(T) = C\cdot(T/\ln(T))^{1/6}$, with $C$ a constant: If between $T_1$ and $T_2$ observations are expected between \CPs, $\mathcal{L} = \{L \in \mathbb{N}: L(T_1) \leq L\leq L(T_2) \}$. In our experiments, we employ this strategy using $T_1 = 1, T_2 = T$. %\mathbb{N} \cap [L(T_1), L(T_2)]$.

\subsection{Modeling spatial dependence}\label{nbh_sequences_section}

While Thm. \ref{VAR_Thm} motivates approximating spatio-temporal processes between \CPs with \eqref{BVAR_eq_1}--\eqref{BVAR_eq_4}, the matrices $\{{\*A}^L_l\}_{l=1}^L$ have $S (LS + 1 + E)$ parameters.
%Even though estimation is computationally cheap due to conjugacy, 
This increases model complexity and ignores spatial information. %The next two sections provide an easy to implement and powerful remedy for both these problems.
We remedy both issues through neighbourhood systems on $\mathcal{S}$. %  inducing sparsity in $\{{A}^L_l\}_{l=1}^L$.
\begin{definition}[\textbf{Neighbourhood system}]
 For a set of locations $\mathcal{S}$ %. Define for all $s \in \mathcal{S}$ 
 with the sets $N_i(s) \subseteq \mathcal{S}$  as the {\normalfont$i$-th neighbourhoods of $s$} for  $0 \leq i \leq n$ and all $s \in \mathcal{S}$, let $N_i(s) \cap N_j(s) = \emptyset$, $s' \in N_i(s) \Longleftrightarrow s \in N_i(s')$ and $N_0(s) = \{ s \}$. Then, the corresponding {\normalfont neighbourhood system} is ${N}(\mathcal{S}) = \left\{\{N_i(s)\}_{i=1}^{n}: s \in \mathcal{S}, 0 \leq i \leq n\right\}$.
\end{definition}
%Naturally, one expects information to be contained in the spatial arrangement of the locations in $\mathcal{S}$: Measurements likely co-vary more the closer they are. This degree of closeness can be defined using the notion of  \textit{spatial neighbourhoods}. 
%For each $s \in \mathcal{S}$, we define the set $N_i(s) \subseteq \mathcal{S}$ with $1\leq i \leq n$ as the $i$-th neighbourhood of $s$. Let the ordering of the indices imply that locations $s' \in N_i(s)$ are closer to $s$ the smaller $i$. The neighbourhoods are also non-overlapping and symmetric, i.e. $N_i(s) \cap N_j(s) = \emptyset$ and $s' \in N_i(s) \Longleftrightarrow s \in N_i(s')$. In the remainder, denote such a neighbourhood systems as $N(\mathcal{S}) = \{ \{N_i(s)\}_{i=1}^{n}: s \in \mathcal{S}\}$. In practice, $N(\mathcal{S})$ can be constructed using distance metrics or clustering procedures.
%The interpretation of these neighbourhood sequences $\{N_i(s)\}_{i=1}^n$ depends on how they are generated. For example, generating them via Euclidean distance in $\mathbb{R}^2$ so that $s' \in N_i(s) \Longleftrightarrow \| s - s' \|_2^2 \in [d_{i-1}, d_i)$ with $d_0 = 0$, then $s' \in N_i(s) \Longleftrightarrow s'$ lies in the space between the concentric circles around $s$ with radius $d_{i-1}$ and $d_i$. 
%Next, we show how this generic framework of modelling spatial dependence can be framed in a way that can be absorbed into BVAR.
In the remainder of the paper, smaller indices $i$ imply that the neighbourhoods $N_i(s)$ are closer to $s$.
%In practice, $N(\mathcal{S})$ can for example be constructed using distance metrics or clustering procedures.
%The neighbourhood system $N(\mathcal{S})$ defines a spatial relationship. 
%To restrict the parameterization of a BVAR model with lag length $L$ based on $N(\mathcal{S})$, 
For a \BVAR model of lag length $L$, the decay of spatial dependence is encapsulated through $\Pi:\{1,\dots,L\} \to \{0,\dots, n\}$. In particular, 
only  $s' \in N_i(s)$ with $i\leq \Pi(l)$ are modeled as affecting $s$ after $l$ time periods. 

%temporal meaning is assigned to $N(\mathcal{S})$ via $p:\{1,\dots,L\} \to \{0,\dots, n\}$ by modeling only  $s' \in N_i(s)$ with $i\leq p(l)$ as affecting $s$ after $l$ time periods. 
%Since $i$ corresponds to spatial closeness, $p(l)$ is decreasing in $l$.
%, reflecting the assumption that spatial dependence wears off more rapidly the further apart two locations are. 
%To absord this generic modelling framework for spatial dependence into the BVAR model family, we need to weave it together with the time dimension:
%Define for a given lag length $L$ the  sequence of neighbourhood sequences $N(\mathcal{S}) = \{ \{N_i(s)\}_{i=1}^{n}: s \in \mathcal{S}\}$ and let $p(\cdot)$ be a function from the lag lengths $l=1,2\dots, L$ into the neighbourhood indices $\{1,2\dots,n\}$.
%The function $p(l)$ is the way by which the neighbourhoods obtain a temporal meaning: For any $1\leq l \leq L$ and $1 \leq k \leq n$, $p(l) = k$ indicates that one assumes the measurements taken at $s' \in N_i(s)$  with $1\leq i \leq k$ at time $t-l$ to have a \textit{non-zero} and those taken at $ s' \in N_j(s)$ with $j>k$ at time $t-l$ to have a \textit{zero} effect on location $s$ at time $t$. For example, the constant function $p(\cdot) = n$ entails that one expects \textit{all} neighbourhoods to have an effect for \textit{all} lag lengths $l:1\leq l \leq L$. In practice, $p(l)$ will usually be a decreasing function of $l$, reflecting the assumption that spatial dependence wears off more quickly in time the further apart two spatial locations are. 
 
	\begin{figure}[h!]
%\begin{subfigure}[b]{4.25cm}
%\begin{minipage}[c]{4.25cm}
	\begin{tikzpicture}	
	%DECLARE STYLES
		\tikzstyle{darkstyle} = [circle,draw,fill=gray!40,minimum size=10]
		\tikzstyle{nbh_one} =[circle,draw,fill=red!40,minimum size=10]
		\tikzstyle{nbh_two} =[circle,draw,fill=blue!40,minimum size=10]
		\tikzstyle{AR} = [circle, draw, fill=orange!40, minimum size = 10]
		\tikzstyle{whitestyle} = [circle, draw, fill=white!40, minimum size = 10]	
	%T-2 nbh
		\node [text width = 1cm] (tmtwo) at (0.85, 2.5) {$t-2$};
		\node [darkstyle]  (1) at (0,1.5) {1\label{1}};
		\node [nbh_one]  (2) at (0.75,1.5) {2\label{2}};
		\node [darkstyle]  (3) at (1.5,1.5) {3\label{3}};
		\node [nbh_one]  (4) at (0,0.75) {4\label{4}};
		\node [AR]  (5) at (0.75,0.75) {5\label{5}};
		\node [nbh_one]  (6) at (1.5,0.75) {6\label{6}};
		\node [darkstyle]  (7) at (0,0) {7\label{7}};
		\node [nbh_one]  (8) at (0.75,0) {8\label{8}};
		\node [darkstyle]  (9) at  (1.5,0) {9\label{9}};
		
	%T-1 nbh
		\node [text width = 1cm] (tmone) at (3.85, 2.5) {$t-1$};
		\node [nbh_two]  (1_) at (3,1.5) {1\label{1_}};
		\node [nbh_one]  (2_) at (3.75,1.5) {2\label{2_}};
		\node [nbh_two]  (3_) at (4.5,1.5) {3\label{3_}};
		\node [nbh_one]  (4_) at (3,0.75) {4\label{4_}};
		\node [AR]  (5_) at (3.75,0.75) {5\label{5_}};
		\node [nbh_one]  (6_) at (4.5,0.75) {6\label{6_}};
		\node [nbh_two]  (7_) at (3,0) {7\label{7_}};
		\node [nbh_one]  (8_) at (3.75,0) {8\label{8_}};
		\node [nbh_two]  (9_) at  (4.5,0) {9\label{9_}};
	
	%T nbh
		\node [text width = 1cm] (t) at (7.2, 2.5) {$t$};
		\node [darkstyle]  (1__) at (6,1.5) {1\label{1__}};
		\node [darkstyle]  (2__) at (6.75,1.5) {2\label{2__}};
		\node [darkstyle]  (3__) at (7.5,1.5) {3\label{3__}};
		\node [darkstyle]  (4__) at (6,0.75) {4\label{4__}};
		\node [whitestyle]  (5__) at (6.75,0.75) {5\label{5__}};
		\node [darkstyle]  (6__) at (7.5,0.75) {6\label{6__}};
		\node [darkstyle]  (7__) at (6,0) {7\label{7__}};
		\node [darkstyle]  (8__) at (6.75,0) {8\label{8__}};
		\node [darkstyle]  (9__) at  (7.5,0) {9\label{9__}};
		
	%connections between nodes of same nbh
		%nbh 1, t-2
		\path (2) edge [red!40, line width = 1.0] (4);
		\path (2) edge [red!40, line width = 1.0] (6);
		\path (4) edge [red!40, line width = 1.0] (8);
		\path (6) edge [red!40, line width = 1.0] (8);
		%nbh 1, t-1
		\path (2_) edge [red!40, line width = 1.0] (4_);
		\path (2_) edge [red!40, line width = 1.0] (6_);
		\path (4_) edge [red!40, line width = 1.0] (8_);
		\path (6_) edge [red!40, line width = 1.0] (8_);
		%nbh 2, t-1
		\path (1_) edge [blue!40, out = 50, in = 130,line width = 1.0] (3_);
		\path (1_) edge [blue!40, out = 225, in = 140, line width = 1.0] (7_);
		\path (3_) edge [blue!40,in = 40, out = -40 , line width = 1.0] (9_);
		\path (7_) edge [blue!40, out = -50, in = -130, line width = 1.0] (9_);
	
	%linear effect connections	
		%nbh 1 connections
	 	\path[->] (2) edge  [out = 10, in =125, red!40, line width = 2.0]  (5__);
	 	\path[->] (2_) edge  [out = 20, in =125, red!40, line width = 2.0]  (5__);
	 	%AR connections
	 	\path[->] (5) edge  [out = -25, in =225, orange!40, line width = 2.0]  (5__);
	 	\path[->] (5_) edge  [out = -50, in =225, orange!40, line width = 2.0]  (5__);
	 	%nbh 2 connection
	 	\path[->] (3_) edge  [out = -40, in =180, blue!40, line width = 2.0]  (5__);
	\end{tikzpicture}
	\vskip -0.2in
%\end{minipage}
\caption{\textit{SSBVAR modeling:} Suppose that on a regular grid of size $9$, %with the $4$- and $8$-neighbourhoods, 
 $Y_{t,5}$  depends on the past two realizations of itself and its $4$- neighbourhood, and the last realization of its $8$-neighbourhood. This is an \SSBVAR on $\mathcal{S}  = \{1, \dots, 9\}$ with $L=2$, {\color{orange}$N_0(5) =\{ 5 \}$}, {\color{red}$N_1(5) = \{2,4,6,8\}$}, {\color{blue}$N_2(5) = \{1,3,7,9\}$} and function $\Pi$ with $\Pi(1) = 2, \Pi(2) = 1$. }
%\end{subfigure}
%\caption{}
\label{SSBVAR_graph}
\end{figure}
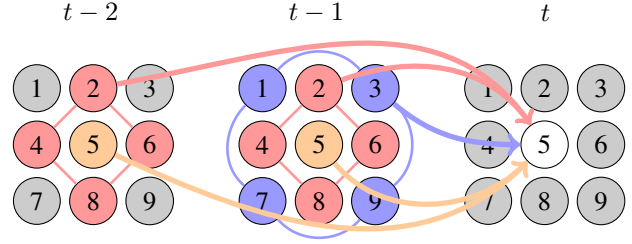

\subsection{Spatializing \BVAR}

In principle, given $N(\mathcal{S})$,  sparsification of the \BVAR model \eqref{BVAR_eq_1}--\eqref{BVAR_eq_4} is possible in two ways: As restriction on the \textit{contemporaneous} dependence via the covariance matrix of the error term $\*\varepsilon_t$, or as restriction on the \textit{conditional} dependence via the coefficient matrices $\{\*A_l\}_{l=1}^L$. We choose the latter for three reasons: 
Firstly, linear effects have more interesting interpretations than error covariances. 
%On a very fundamental level, dependence in the error term is exactly that: The error, i.e. the dependence we could not meanigfully model with the observations we are given. Concordantly, it is also the part of the equation that -- if possible -- should not be at the center stage of interpreting dependence patterns. 
Secondly, using $\{\*A_l\}_{l=1}^L$ to encode spatial dependency allows us to work with arbitrary neighbourhoods. In contrast, modelling dependent errors under conjugacy limits dependencies to decomposable graphs \cite{MurphyMVTSBCP}. Since not even a regular grid is decomposable, this is problematic for spatial data. 
Thirdly, modelling errors as contemporaneous is attractive for low-frequency data where the resolution of temporal effects is coarse, but the situation reverses for high-frequency data. Since the algorithm runs on-line, we expect $\{\*Y_t\}$ to be observed with high frequency.
%with high-frequency data is unattractive, as 
% while modelling dependency via simultaneous errors is interesting for low-frequency data, the conditional approach is more natural for high-frequency data. As the algorithm is designed to run on-line, it is reasonable to expect $\{Y_t\}$ to be observed with high frequency.
%Thirdly,  for on-line problems the frequency in which data is observed is usually high enough to capture all dependency due to truly simultaneously occuring errors by the most recent past.
%Motivating that conditional modelling approach more interesting for on-line problem since frequency of data is high
%For example, if one had a regular grid of observations in $\mathbb{Z}^2$, requiring decomposability would prohibit most models suitable to model spatial dependence, even very simple dependence structures like the $4$-neighbourhood given by $N((s_1, s_2)) = \{(s_1-1, s_2-1), (s_1-1, s_2+1), (s_1+1, s_2-1), (s_1+1, s_2+1)\}$.
%Embracing the advantages of conditional modelling, we generate sparsity by spatially structuring the BVAR of equations \eqref{BVAR_eq_1}--\eqref{BVAR_eq_4}.
%
\begin{definition}[\textbf{Spatially structured \BVAR (\SSBVAR)}]
 For process $\{\*Y_t\}$ on $\mathcal{S}$ and $(L, N(\mathcal{S}), \Pi(\cdot))$, define the matrices $\{\widetilde{\*A}_l\}_{l=1}^L$ by imposing that $[\widetilde{\*A}_l]_{(s, s')} = 0 \Longleftrightarrow s' \notin N_i(s)$ for any $i \leq \Pi(l)$. Let $\widetilde{\*A}_l^{\neq 0}$ be the vector of non-zero entries in $\widetilde{\*A}_l$ and $\widetilde{\*c} = (\*\alpha, \text{vec}(\*B),\widetilde{\*A}_1^{\neq 0},  \widetilde{\*A}_2^{\neq 0}, \dots  \widetilde{\*A}_L^{\neq 0})^T$. The \SSBVAR model on $\{\*Y_t\}$ induced by $(L, N(\mathcal{S}), \Pi(\cdot))$ is obtained by combining \eqref{BVAR_eq_1}--\eqref{BVAR_eq_2} with
 \begin{IEEEeqnarray}{rCl}
  \IEEEyessubnumber*
 	\widetilde{\*c}|\sigma^2 & \sim & \mathcal{N}(\*0, \sigma^2 \cdot \*V_{\widetilde{c}})\label{BVAR_eq_3_prime} \\
 	\*Y_t & = & \textstyle \*\alpha + \*B \*Z_t + \sum_{l = 1}^L \widetilde{\*A}_l \*Y_{t-l} + \*\varepsilon_t. \label{BVAR_eq_4_prime} \\[-15pt] \nonumber
 \end{IEEEeqnarray}
\end{definition}

% The entry $(s,s')$ of matrix $A_l$ is assumed to be non-zero if and only if $s' \in N_i(s)$ for some $i \leq p(l)$. Denote these transformed matrices by $\widetilde{A}_l$,  and stack the non-zero entries of  $\widetilde{A}_l$ into a vector $\widetilde{A}_l^{\neq 0}$. Similarly, denote $\widetilde{c} = (\alpha, \text{vec}(B),\widetilde{A}_1^{\neq 0},  \widetilde{A}_2^{\neq 0}, \dots  \widetilde{A}_L^{\neq 0})^T$. 
%Keeping equations \eqref{BVAR_eq_1}--\eqref{BVAR_eq_2} unchanged, one can now spatialize the equations \eqref{BVAR_eq_3}--\eqref{BVAR_eq_4} as

%In the remainder of this paper, this type of model takes center stage and is called a \textit{spatially structured BVAR (SSBVAR)} model generated by the triplet $(L, N(\mathcal{S}), p(\cdot))$.
%Since it is conjugate, an SSBVAR model generated by any $(L, N(\mathcal{S}), p(\cdot))$ defines a valid and computationally cheap model that can be used inside the algorithm as part of the model universe $\mathcal{M}$.

Fig. \ref{SSBVAR_graph} illustrates this idea. Further sparsification is possible by modelling neighbourhoods jointly, i.e. $[\widetilde{\*A}_l]_{(s, s')} = a_i(s), \forall s' \in N_i(s)$, reducing the number of parameters to $S \cdot \sum_{l=1}^L\Pi(l)$. If one imposes $a_i(s) = a_i(s') = \dots = a_i$, this number drops to $\sum_{l=1}^L\Pi(l)$.

\subsection{Building \SSBVARs: choosing $L, N(\mathcal{S}), \Pi(\cdot)$}

%This section addresses how to choose the triplet $(L, N(\mathcal{S}), p(\cdot))$. 
%
For the choice of lag lengths $L$, part \textit{(3)} of Thm. \ref{VAR_Thm} suggests  $L \in \{L' \in \mathbb{N}: L(T_1) \leq L'\leq L(T_2) \}$ if one expects $T_1$ to $T_2$ observations between \CPs. 
%
%HIGHLIGHT HOW WE CAN ONLY USE COMPETING NBHS WITH NEW ALGO + WHY %IT IS GREAT
%
%Give examples using road distance, euclidean distance, ... point to something Make sure that we hint at generality of nbh systems
%
For any data stream $\{\*Y_t\}$ on a space $\mathbb{S}$, there are different ways of constructing neighbourhood structures $N(\mathcal{S})$.  For example, when analysing pollutants  in London's air  in section \ref{results}, $N(\mathcal{S})$ could be constructed from Euclidean or Road distances.
%In most inference frameworks built on neighbourhood structures, one would have to decide for a single one, and thus think long and hard about this issue. 
By filling $\mathcal{M}$ with \SSBVARs constructed  using competing versions of $N(\mathcal{S})$, \BOCPDMS provides a way of dealing with such uncertainty about spatial relations. In fact, it can dynamically discern changing spatial relationships on $\mathbb{S}$.
Lastly, $\Pi(\cdot)$ should usually be decreasing to reflect that measurements affect each other less when further apart.
%Again, multiple reasonable functions can be paired with multiple reasonable neighbourhood structures to let the algorithm choose which one fits best. 

\section{Hyperparameter optimization}\label{hyperpar_opt}

Hyperparameter inference  on $\*\nu_m$ can be addressed either by 
%Inference depends on hyperparameters $\theta_m^0$. This can be adressed either by 
introducing an additional hierarchical layer \cite{HazardLearningBOCD} or using type-II \ML. The latter is obtained by maximizing the model-specific evidence % for each  model $m \in \mathcal{M}$ the evidence %$\theta_m^0$ for each model $m \in \mathcal{M}$ can be efficiently obtained by maximizing the evidence
\begin{IEEEeqnarray}{rCl}\label{hyperpar_opt_Turner}
	\log p(\*y_{1:T}|\*\nu_m) 
	& = & 
	\sum_{t=1}^T \log p(\*y_t|\*\nu_m, \*y_{1:(t-1)}). 
\end{IEEEeqnarray}
Computation of the righthand side requires evaluating the gradients $\nabla_{\*\nu_m} p(\*y_{1:t}, r_t| \*\nu_m)$, which are obtained efficiently and recursively %section \ref{Algorithm} 
\citep{ABOCD}. %, see \citet{CaronDoucet}.  
\citet{GPBOCD} use  $\*y_{1:T'}$ as a test set, and %employ conjugate gradient descent by 
run \BOCPD $K$ times to find $\widehat{\*\nu}_m = \arg\max_{\*\nu_m}\left\{ p(\*y_{1:T'}|\*\nu_m) \right\}$. 
{
Most other on-line \GP approaches also
% dealing with hyperparameters typically is not easy and requires substantial recomputations
 require substantial recomputations for hyperparameter learning
\citep[e.g.,][]{OnlineGP2}.
}
In contrast, \citet{CaronDoucet} propose on-line gradient descent updates via
\begin{IEEEeqnarray}{rCl}%\label{hyperpar_opt_Caron}
	\*\nu_{m,t+1} 
	& = & 
	\*\nu_{m,t} + \alpha_t  \nabla_{\*\nu_{m,t}}\log p(\*y_{t+1}| \*y_{1:t}, \*\nu_{m_{1:t}}).  \quad\; \label{gradient_descent_caron}
\end{IEEEeqnarray}
The latter  is preferable %to the method of  \citet{ABOCD}
 for two reasons:
%than that of \citet{ABOCD}, but has two major advantages: 
Firstly, inference and type-II \ML are executed simultaneously (rather than sequentially) and thus enable cold-starts of \BOCPDMS. 
%In other words, it allows for a cold-start without needing to choose parameters a priori. 
Secondly, neither the on-line nature nor the computational complexity of \BOCPDMS is affected.

%Thirdly, though the off-line method of \citet{ABOCD} is more robust, it did not increase performance significantly.
% in performance when compared to \citet{CaronDoucet}.

\section{Computation \& Complexity}\label{computational_complexity}

%Motivating that conditional modelling approach more interesting for on-line problem since frequency of data is high

%Models stupidly parallel (step 1 in pseudocode), this is where all the heavy lifting is

%Step 2 in pseudo-code really cheap (just sums)
While tracking $|\mathcal{M}|$ models, \BOCPDMS has linear time complexity. 
Step 1 in the pseudocode  is the bottleneck, but looping over $\mathcal{M}$ can be parallelized: With $N$ threads, it executes in $\mathcal{O}\left(\lceil |\mathcal{M}|/N \rceil \cdot \max_{M \in \mathcal{M}} \text{CmpTime}(M)\right)$. Step 2 takes $\mathcal{O}(|R(t)||\mathcal{M}|)$, for $R(t)$ all run-lengths at time $t$.

\subsection{Pruning the run-length distribution}\label{RLD}
In a naive implementation, all run-lengths are retained and $R(t) = \{1,2,\dots,t\}$.  %Letting $C$ be some time-independent constant,
 This implies execution time of order $\mathcal{O}(t)$ for processing $\*y_t$, but can be made time-constant by pruning: % to make this operation time-constant. 
 %, resulting in a time-constant execution time $\mathcal{O}(C)$ at time $t$. 
{
If one discards run-lengths whose posterior probability is $\leq 1/R_{\max}$ or only keeps the $R_{\max}$ most probable ones, $|R(t)| \leq R_{\max}$ \cite{BOCD}. A third way is {Stratified Rejection Control (\SRC)} \cite{FearnheadOnlineBCD}, which \citet{CaronDoucet} and the current paper found to perform as well as the other approaches. In our experiments, we prune by keeping the  $R_{\max}$ most probable model-specific run-lengths $p(r_t|m_t,\*y_{1:t})$ for each model. 
}
% As is also suggested in \citet{CHAMP},

%Within each model
\subsection{\BVAR updates}
%Next, we show that if all models in $\mathcal{M}$ are SSBVAR, $\mathcal{O}\left(\max_M \text{CompTime}(M)\right) = \mathcal{O}(R(t)\cdot \min\{k^3, S^3\})$.

The bottleneck when updating a \BVAR model %generated by $(L, N(\mathcal{S}), p(\cdot))$ 
%
%Change name or leave out completely
%
in $\mathcal{M}$ is step I.B.1 in the pseudocode of \BOCPDMS, when updating the \MAP estimate $\*c(r,t) = \*F(r,t)\*W(r,t)$ of the coefficient vector , where $\*F(r,t) = (\*X_{(t-r):t}'\*X_{(t-r):t} + \*V_{\widetilde{c}})^{-1}$ and $\*W(r,t) = \*X_{(t-r):t}'\*Y_{(t-r):t}$ for all $r \in R(t)$. %Clearly, 
Since $\*W(r,t) = \*W(r-1,t-1) + \*X_t'\*Y_t$, updates are $\mathcal{O}(kS)$. $\*F(r-1,t-1)$ can be updated to $\*F(r,t)$ using rank-$k$ updates to its \QR-decomposition in $\mathcal{O}( k^3)$ or using Woodbury's formula, in $\mathcal{O}(S^3)$, implying an overall complexity of $\mathcal{O}(|R(t)| \min\{k^3, S^3\})$ at time $t$.

% and requires computing $\widehat{c}(r,t) = M(r,t)J(r,t)$ for all  run lengths $r \in R(t)$ retained at time $t$, where for $X_{(t-r):t}$ being the $(t+1-r)\cdot S \times k$ regressor matrix at time $t$,
%\begin{IEEEeqnarray}{rCl}
% M(r,t) &=& (X_{(t-r):t}'X_{(t-r):t} + V_{\widetilde{c}})^{-1} \\
% J(r,t) & = & X_{(t-r):t}'Y_{(t-r):t}.
%\end{IEEEeqnarray}
% $M(r,t) = (X_{(t-r):t}'X_{(t-r):t} + V_{\widetilde{c}})^{-1}$,
% $J(r,t) = X_{(t-r):t}'Y_{(t-r):t}$.
%Clearly, $J(r,t) = J(r-1,t-1) + X_t'Y_t$, allowing updates in $\mathcal{O}(kS)$ time. 
%$M(r-1,t-1)$ can be updated to $M(r,t)$ using rank-$k$ updates to its QR-decomposition ($\mathcal{O}( k^3)$) or using Woodbury's formula, ($\mathcal{O}(S^3)$), implying a complexity of $\mathcal{O}(|R(t)| \min\{k^3, S^3\})$ at time $t$. 
%Steps 1B.2 and 1B.3 have time complexity of $\mathcal{O}(|R(t)|)$ regardless of the family the updated model belongs to. 

\subsection{Comparison with \GP-based approaches}
Define %$k_M$ as the number of regressors for SSBVAR models $M \in \mathcal{M}$ and 
$k_{\max}$ as the largest number of regressors of any \BVAR model in $\mathcal{M}$. % = \max_{M \in \mathcal{M}} k_M$. 
From the previous paragraphs, it follows that if all models in $\mathcal{M}$ are \BVARs, the overhead $C =\lceil N/|\mathcal{M}| \rceil \cdot \min\{k_{\max}^3, S^3\}$ is time-constant. Thus, 
%if the run-length distribution is pruned, 
\BOCPDMS runs in $\mathcal{O}(T R_{\max})$ on $T$ observations. In contrast, the models of \citet{GPBOCD} run in $\mathcal{O}(T R_{\max}^3)$. %Even taking into account the constant $C$, this implies that for $R_{\max} = 100$, it can process a more than $20$-dimensional data stream faster than the \GP-model a univariate one. 
{
The experiments in section \ref{results} confirm this:
Using the software of \citet{TurnerThesis} on the Nile data, fitting one \ARGPCP model  takes $42$ seconds compared to $12$ seconds when fitting three models in \BOCPDMS, so a \BVAR model is $>10\times$ faster to process. Per inferred parameter, \BOCPDMS is $>60\times$ faster than \ARGPCP; and this factor is much larger for multivariate data (e.g., $>270$ for the $30$ Portfolio data). More details on the run-times can be found in the Appendix.
}

%\textcolor{blue}{POINT TO APPENDIX}
%GIVE SOME NUMERICAL VALUES FOR PF DATA

%implying that when the run-length distribution is pruned and one defines $C =\lceil N_t/|\mathcal{M}| \rceil \cdot \min\{k_{\max}^3, S^3\}$, the algorithm takes at most $\mathcal{O}(T R_{\max}C)$ time to execute on $T$ observations. In contrast, the models proposed in \citet{GPBOCD} are of order $\mathcal{O}(T R_{\max}^3)$ for a different constant $C'$. 

%place PF pics here
\begin{figure*}[t!]
%\vskip 0.2in
\begin{center}
\centerline{\includegraphics[trim={4.6cm 1cm 5.0cm 0.3cm},clip, width=1.00\textwidth]{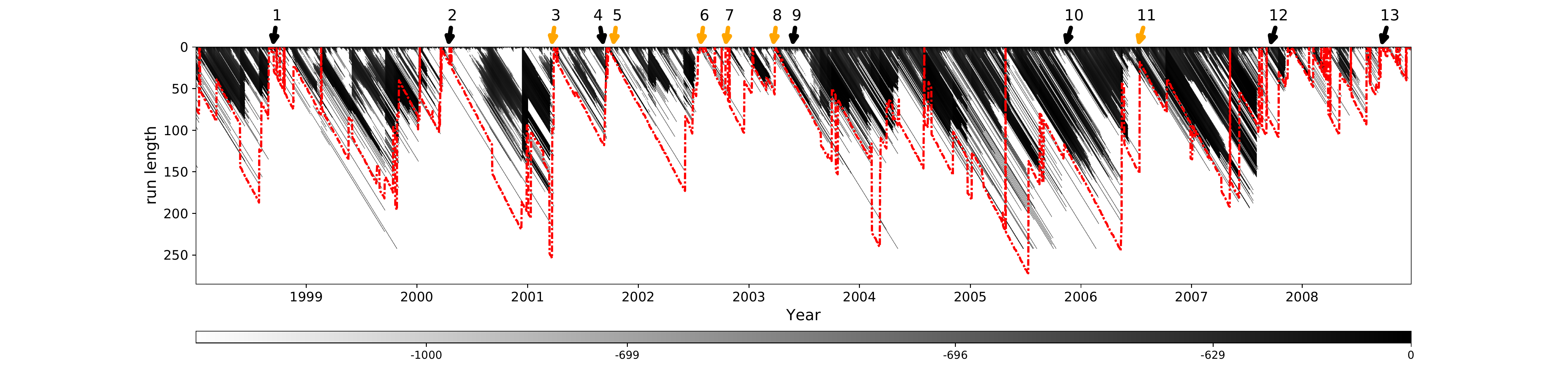}}
%left lower right upper
\caption{\textit{Results for $30$ Portfolio data set, displayed from 01/01/1998--31/12/2008}: Log run-length distribution  (grayscale) and its {\color{red}maximum} (dashed). Changepoints (\CPs) found by \citet{GPBOCD} are marked in \textbf{black}, additional \CPs found by \BOCPDMS in {\color{orange}orange}. Labels correspond to: 
(1) Asia Crisis, 
(2) DotCom bubble bursting, 
(3) OPEC cuts output by 4\%,
(4) 9/11, 
(5) Afghanistan war, 
(6) 2002 stock market crash, 
(7) Bombing attack in Bali,
(8) Iraq war, 
(9) Major tax cuts under Bush, 
(10) US election, 
(11) Iran announces successful enrichment of Uranium,
%(11) Record level of Dow Jones (bull market of 2006), 
(12) Northern Rock bank run, 
(13) Lehman Brothers collapse.}
\label{30PF_comparison}
%\label{picture_demo}
\end{center}
\vskip -0.2in
\end{figure*}
%
%place PF pics here
\begin{figure*}[]
%\vskip 0.2in
\begin{center}
%^4.8cm 1cm 5.0cm 0.3cm
\centerline{\includegraphics[trim={4.6cm 1.025cm 5cm 3.0cm},clip, width=1.00\textwidth]{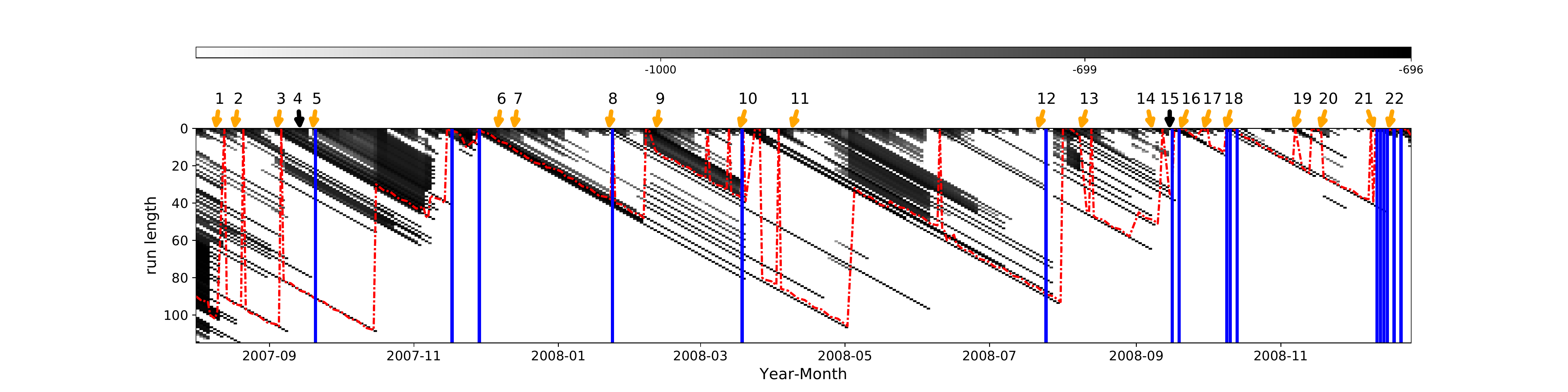}}
%left lower right upper
\caption{\textit{Financial crisis 01/08/2007--31/12/2008}: Colours as in Fig \ref{30PF_comparison}, with {\color{blue}\MAP segmentation}. %Changepoints (CPs) with ground truth found by \citep{GPBOCD} are marked in black, additional CPs with ground truth in green. 
Event labels: 
(1) BNP Paribas funds frozen, 
(2) Fed cuts lending rate, %by 0.5\% points, 
(3) IKB 1bn\$ losses, 
(4) Northern Rock bank run, 
(5) Fed cuts interest rate, %BoE \textsterling10bn %by 0.5\% points, BoE injects  \textsterling10bn into market 
(6) Bush %announces 
	rescue plan for \textgreater $10^6$ homeowners, % facing foreclosure, 
(7) Fed, ECB, BoE loans for banks, 
(8) Fed cuts funds rate, % to $3.5\%$, 
(9) G7 estimate: 400bn\$ losses worldwide, 
(10) JP Morgan buys Bear Stearns, 
(11) IMF estimate: \textgreater 1trn\$ losses worldwide,
(12) HBOS' rights issue fails, 
(13) ECB provides \euro{}200bn for liquidity, 
(14) Fannie Mae \& Freddie Mac bailout, 
(15) Lehman collapse, 
(16) Russia: 500bn Roubles crisis package, %BoE provides \textsterling10bn, 
(17) Fortis bailout, % with 11.8bn\$, BENELUX
(18) UK: \textsterling500bn bank rescue package, 
(19) %IMF lends 16.4bn\$ to Ukraine, 
	BoE, ECB cut interest rate, % rates further, 
(20) G20 promise fiscal stimuli, 
(21) Madoff's Ponzi scheme revealed, South Korean %central bank 
 CB sets interest rate at record low
(22) Fed, Japanese central bank cut interest rates. % to $0.25$\%,  Japanese central bank cuts interest rate to $0.3$\%.
 Dates from \citet{timeline}.}
\label{30PF_FinCrisis}
%\label{picture_demo}
\end{center}
\vskip -0.2in
\end{figure*}

\section{Experimental results}\label{results}

We evaluate the performance of the algorithm in two parts. First, we compare it to benchmark performances of \GP-based models on real world data reported by \citet{GPBOCD}.  
This shows that as implied by Thm. \ref{VAR_Thm}, \VARs are
excellent approximations for a large variety of data streams.
%This shows the power of Thm. \ref{VAR_Thm}, namely that \VARs are excellent approximations for a large variety of data streams. %, even when compared with models that are much more computationally complex and parameter-rich. 
Next, we showcase \BOCPDMS' novelty in the  multivariate setting. % and unique capabilities. % benefits when coupled with SSBVAR models.
All computations can be reproduced with code available on the first author's website.  We use uniform model priors $q$, a constant Hazard functions $H$ and gradient descent for hyperparameter optimization as in Section \ref{hyperpar_opt}. The lag lengths of models in $\mathcal{M}$ are chosen based on Thm. \ref{VAR_Thm} (3) and the rates of \citet{HannanKavalieris} for  \BVARs and Bayesian Autoregressions (\BARs), respectively.

%MAKE CLEARER THAT NUMBERS COME FROM GP PAPER

\subsection{Comparison with \GP-based approaches}
%For benchmarking the different CP models of their algorithm,
As in \citet{GPBOCD}, \ARGPCP will refer to the non-linear \GP-based \AR model, \GPTSCP to the time-deterministic model, and \NSGP to the non-stationary \GP allowing hyper-parameters to change  at every \CP. 
\citet{GPBOCD} compute the mean squared error (\MSE) as well as the negative log predictive likelihood (\NLL) of the one-step-ahead predictions  %based on Eq.  \eqref{posterior_predictive} 
 for three data sets: The water height of the Nile between $622-1284$ AD, the snowfall in Whistler (Canada) over a 37 year period and the $3$-dimensional time series ($x$-, $y$-coordinate and headangle) of a honey bee during a waggle dance sequence. In \citet{TurnerThesis}, all of the models except \NSGP were also compared on daily returns for $30$ industry portfolios from $1975-2008$.
 %\footnote{Unfortunately, we cannot compare the NLL on this data set, since the values reported in \citet{TurnerThesis} were recentered so that their best performing method's NLL was set to $0.0$.}
In Table \ref{benchmark_table}, \BOCPDMS is compared to these benchmarks  for $\mathcal{M}$ consisting of \BAR and \SSBVAR models.  
%As in \citet{GPBOCD}, ARGPCP refers to the non-linear GP-based AR model, GPTSCP to the time-deterministic model, and NSGP to the non-stationary GP allowing hyper-parameters to change  at every CP. 
%The data sets are the water height of the river Nile between $622-1284$ CE, the snowfall in Whistler (Canada) over a 37 year period, the 3-dimensional waggle dance sequence of a honey bee and daily returns for $30$ industry portfolios from $1975--2008$.
%Table \ref{benchmark_table} shows the results of this comparison.\footnote{Matching the original simulations (see https://sites.google. com/site/wwwturnercomputingcom/software), the Nile and bee waggle data were normalized and the Whistler data was log transformed. The inverse-normal transform was used for the Portfolio data as described in \citep[p.169]{TurnerThesis}, but without whitening it first, since the latter takes out the correlation across portfolios.} %For all data sets, the model universe $\mathcal{M}$ is populated only with AR and BVAR models. Their lag lengths are chosen using the optimal rates of \citet{HannanKavalieris} and part \textit{(3)} of Theorem \ref{VAR_Thm}, respecitvely.
%As in \citet{GPBOCD}, ARGPCP refers to the non-linear GP-based AR model, GPTSCP to the time-deterministic model, and NSGP to the non-stationary GP allowing hyper-parameters to change  at every CP. 

\begin{table}[ht]
{\renewcommand{\arraystretch}{1.2}
\caption{One-step-ahead predictive \MSE and \NLL of \BOCPDMS compared to \GP-based techniques, with $95\%$ error bars. All \GP results are taken from \citet{GPBOCD} and \citet{TurnerThesis}.}
% \NLL marked $^{\ast}$ are centered in  \citet{TurnerThesis}  so that the \NLL of the best-performing \GP-method amongst them is $0.0$.}
%Use some dots or hearts to separate table entries
\label{benchmark_table}
%\vskip 0.1in
\begin{center}
\vskip 0.1in
\begin{small}
\begin{sc}
\begin{tabular}{ p{1.2cm} p{1.2cm} p{1.2cm} p{1.2cm} p{1.25cm}}
 \hline\hline
  \multicolumn{1}{c}{}& \multicolumn{2}{c}{Nile}& \multicolumn{2}{c}{Snowfall}\\
  \hline\hline
 Method     & \MSE 	& \NLL & \MSE 	& \NLL \\[1.5pt]
 \hline
 \ARGPCP   	& $0.553 $ 		& $1.15$		&  $0.750$   	& $-0.604$		\\
 			&  $(0.0962) $ 	& $(0.0555) $	&  $(0.0315)$   &  $(0.0385)$		\\
 \GPTSCP 	& $0.583$     	&  $1.19 $ 		& $0.689$    	 &  $1.17 $\\
 			& $(0.0989)$    &  $(0.0548) $ 	& $(0.0294)$     &  $(0.0183) $\\
 \NSGP	 	&  $0.585$   	& $1.15$ 		&  $\mathbf{0.618}$   	& 	$\mathbf{-1.98}$	\\
 			&  $(0.0988)$   & $(0.0655)$ &  $(0.0242)$   	& 	$(0.0561)$	\\
 \BVAR		&	$\mathbf{0.550}$  	&	$\mathbf{1.13}$		& $0.681$    & $0.923$		\\
 			&	$(0.0948)$  	&	$(0.0684)$		& $(0.0245)$    & $(0.0231)$		\\
 \hline\hline
  \multicolumn{1}{c}{}& \multicolumn{2}{c}{Bee Dance}& \multicolumn{2}{c}{30 Portfolios}\\
 \hline\hline
 Method     & \MSE 	& \NLL & \MSE 	& \NLL \\[1.5pt]
 \hline
 \ARGPCP   	& $2.62$ 	& $4.07$		&  $29.95$   & $39.55$		\\
 			& $(0.195)$ & $(0.150)$   	&  $(0.50)$	 & $(0.22)$ 	\\ 
 \GPTSCP 	& $3.13$    & $4.54$		&  $30.17$   & $\mathbf{39.44}$		\\
 			& $(0.241)$ & $(0.188)$		&  $(0.51)$  & $(0.22)$		\\	
 \NSGP	 	& $3.17$    & $4.19$		&  --   & 	--	\\
 			& $(0.230)$ & $(0.212)$		&  --   & 	--	\\
 \BVAR		&$\mathbf{1.74}$ & $\mathbf{3.57}$	& $\mathbf{25.93}$  
 			& $48.32$ 		\\
 			& $(0.222)$    & 	$(0.166) $	& $(0.906)$    & $(0.964)$		\\
 \hline\hline
\end{tabular}
\end{sc}
\end{small}
\end{center}}
\end{table}

\subsubsection{Designing $\mathcal{M}$}

%APPEARS TO BE UNCLEAR (FORMULATION)

%Lag lengths are chosen using the optimal rates of \citet{HannanKavalieris} for BARs,  and part \textit{(3)} of Theorem \ref{VAR_Thm} for BVARs, with $T_1 = 1$, $T_2 = T$. 
{
Both the Nile and the snowfall data are univariate, so $\mathcal{M}$ consists of \BARs with varying lag lengths.
%neither the Nile nor the snowfall data admit a neighbourhood system and are analysed with BAR models only. 
For the $3$-dimensional bee data, $\mathcal{M}$ additionally contains unrestricted \BVARs. 
Lastly, \SSBVARs  are used on the $30$ Portfolio data. Two neighbourhood systems are constructed from distances in the spaces of pairwise contemporaneous correlations and autocorrelations {prior} to $1975$, a third using the \textit{Standard Industrial Classification (\SIC)}, with $\Pi(\cdot)$ decreasing linearly. 
}

\subsubsection{Findings}

\textbf{Predictive performance and fit:} In terms of \MSE, \BOCPDMS clearly outperforms all \GP-models on multivariate data. Even on univariate data, the only exception to this is the snowfall data, where \NSGP does better. 
However, \NSGP %comes at the cost of having to do 
requires grid search or Hamiltonian Monte Carlo sampling for hyperparameter optimization at each observation \cite{GPBOCD}.
%Note that we outperform the non-linear autoregressive ARGCP model on all data sets, because being able to change lag lengths between CPs is more important to predictive performance than being able to model non-linear dynamics. 
Overall, there are three main reasons why \BOCPDMS performs better: Firstly, being able to change lag lengths between \CPs seems more important to predictive performance than being able to model non-linear dynamics. 
Secondly, unlike the \GP-models, we %can model dependence by incorporating information from 
allow the time series to communicate via $\{\*A_l^L\}$. %it for multivariate data it pays to model the we use information from other time series. 
%In particular, both in the bee dance and the $30$ portfolios, the algorithm selects such models.
%Though BOCPDMS significantly outperforms the competing models in terms of MSE, NLL are typically smaller for GP-models. This seeming contradiction dissolves if one recalls the well-known bias-variance decomposition \citep[e.g.][]{Hyndman}:
%\begin{IEEEeqnarray}{rCl}
%	\text{MSE}(\widehat{y}_t) & = & \text{Bias}(\widehat{y}_t)^2 + %\text{Var}(\widehat{y}_t). \label{MSE_decomposition}
%\end{IEEEeqnarray}
Thirdly, the hyperparameters of the \GP have a strong influence on inference. In particular, the noise variance $\sigma$ is treated as a hyperparameter and optimized via type-II \ML. Except for the \NSGP, this is only done during a training period. Thus, the \GP-models cannot adapt to the observations after training, leading to overconfident predictive distributions that are too narrow \citep[see][p. 172]{TurnerThesis}. This in turn leads them to be more sensitive to outliers, and to mislabel them as \CPs. 
In contrast, \eqref{BVAR_eq_1}--\eqref{BVAR_eq_4} models $\sigma$ as part of the inferential Bayesian hierarchy, and hyperparameter optimization is instead applied at one level higher.
% $a,b$.
Consequently, our predictive distributions are wider, and the algorithm is less confident about the next observations, making it more robust to outliers.
This is also responsible for the overall smaller standard errors of the \GP-models in Table \ref{benchmark_table}, since the \GPs interpret outliers as \CPs and immediately adapt to short-term highs or lows.
%Another contributing factor to this is that GPs 
%We find that Gs yield more CPs => WHY? narrower confidence/overconfident => increases MSE, but lowers s.e. (see nile data with short outlier-periods that one can fit better)
%Another contributing factor may be that the GP-models do not inf
%CP models overfit \citep[as noted in][p. 172]{TurnerThesis} and %thus have a large bias . 
%One of the main reasons for this is that the variance parameter $\sigma$ is treated as fixed in the GP-models and has to be optimized via type-II MLE. In contrast, %our type-II MLE optimizes the hyperparameters $a,b$ of $\sigma^2$,
%  \eqref{BVAR_eq_1}--\eqref{BVAR_eq_4} models $\sigma$ as part of the Bayesian hierarchy as an inverse gamma random variable on which posterior inference is performed, and our type-II MLE is instead applied to $\sigma$'s \textit{hyperparameters} $a,b$.
%Consequently, our predictive distributions are wider, and the algorithm is less confident about the next observations, which makes the NLL bigger. In turn, this makes declaring CPs more expensive, robustifying the algorithm against outliers. 

\textbf{\CP Detection:} A good demonstration of this finding is the Nile data set, where the \MAP segmentation finds a single \CP, corresponding to the installation of the nilometer around $715$ CE, see Fig \ref{Nile}. In contrast, \citet{GPBOCD} report $18$ additional \CPs corresponding to outliers. 
The same phenomenon is also reflected in the run-length distribution (\RLD): While the probability mass in Figs. \ref{30PF_comparison},  \ref{30PF_FinCrisis} and \ref{Nile} are spread across the retained run-lengths, the \RLD reported in \citet{GPBOCD} is more concentrated and even degenerate for the $30$ Portfolio data set.
%, meaning the GP-models are overly confident about both the next observation as well as the current run-length. 
On the other hand, such enhanced sensitivity to change can be advantageous. % if one is more concerned about CPs than prediction: 
For instance, in the bee waggle dance, %($17$ CPs in $1057$ observations),
the \GP-based techniques are better at identifying the true \CPs. The reason is twofold: 
%, is there a comma here?
Firstly, the variance for the bee waggle data is homogeneous across time, so treating it as fixed helps inference. Secondly, the \CPs in this data set are subtle, so having narrower predictive distributions is of great help in detecting them.
However, it adversely affects performance when changes in the error variance are essential, as for financial data:
%As expected, CP detection with our method performs best on data high-dimensional enough to admit neighbourhood representations: 
In particular, \BOCPDMS finds the ground truths  labelled in  \citet{GPBOCD}, and discovers even more, see Fig. \ref{30PF_comparison}. 
This is especially apparent in times of market turmoil where changes in the variance of returns are significant. We show this using the example of the subprime mortgage financial crisis: %$01/08/2007-31/12/2008$:  
While the \RLD of \citet{GPBOCD} identified only $2$ \CPs with ground truth and a third unlabelled one during the height of the crisis, \BOCPDMS detects a large number of \CPs corresponding to ground truths, see Fig. \ref{30PF_FinCrisis}.

Lastly, we note that segmentations obtained off-line for both the bee waggle dance and the $30$ Portfolios are reported in \citet{MurphyMVTSBCP}. Compared to the on-line segmentations produced by \BOCPDMS, these are closer to the truth for the  bee waggle data, but not for the $30$ Portfolio data set.
%and results in an extremely high resolution: 
%To showcase the high resolution of our algorithm's outputs, 30PF FIGURE FIN CRIS zooms into the peak of the financial crisis  $01/08/2007-31/12/2008$. While the RLD of \citep{GPBOCD} identifies only $2$ CPs with ground truth and a third one without in that period, we can discern a signficiantly larger number of significant historical events using the MAP segmentation as well as the RLD. 
%bee waggle + nile show that we detect less CPs (due to overfit discussed above), which is a good thing for Nile and not so great for bee waggle, but the algorithm still produces a better fit!
%PF: Show RLD + explain stochasticity  and MP

%
\begin{figure}[b!]
%\vskip 0.2in
\begin{center}
\centerline{\includegraphics[trim={1.08cm 1.08cm 2.0cm 1.5cm},clip, width=1.0\columnwidth]{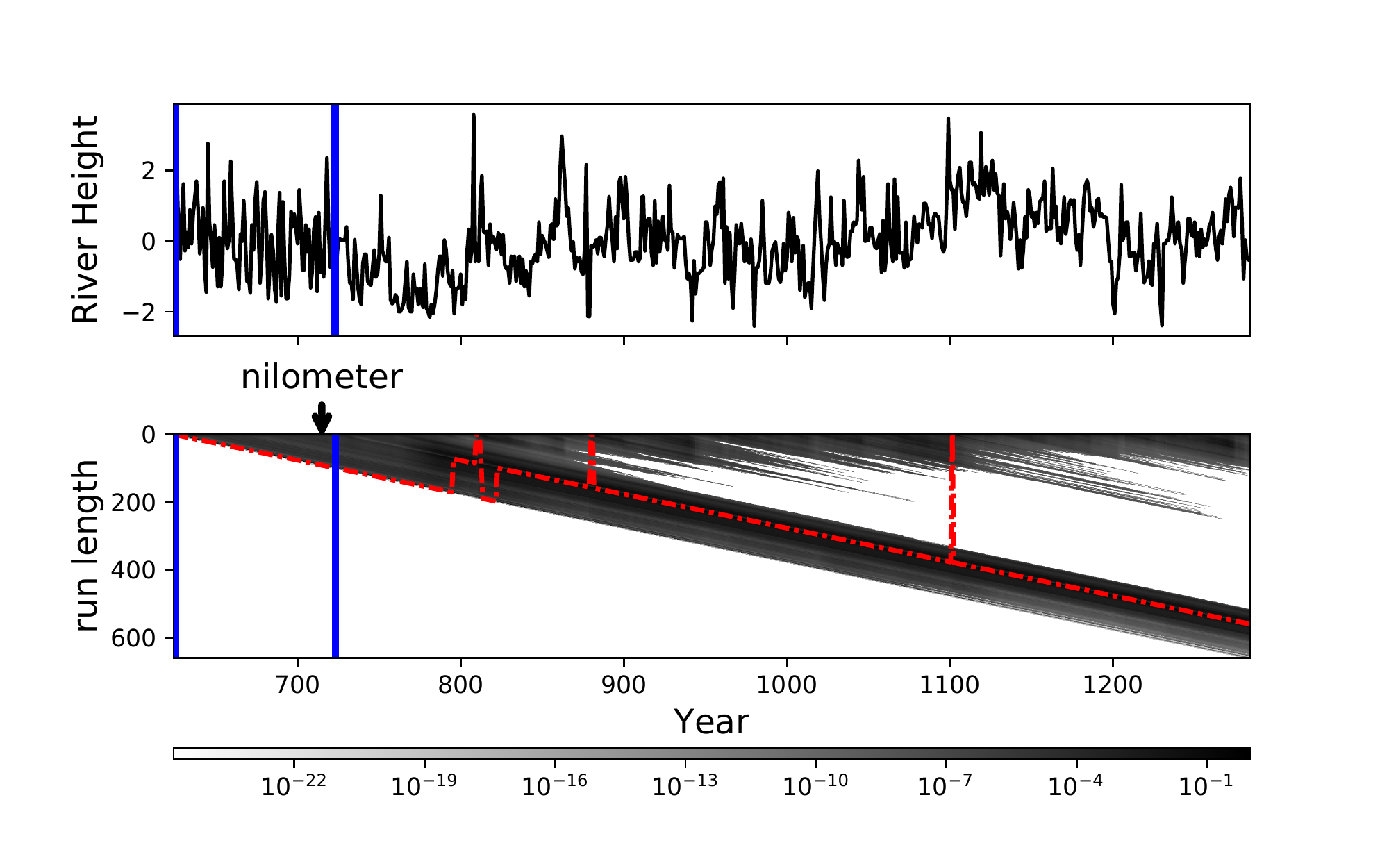}}
%left lower right upper
%DO: Check if we can make textbf a little less fat
\caption{\textit{Results for Nile data}: {\color{darkgray}\textbf{Panel 1:}}  Nile data with structural change at $715$. {\color{darkgray}\textbf{Panel 2:}} Both run-length distribution (grayscale with dashed \textcolor{red}{maximum}) and  {\color{blue}\MAP segmentation} detect the change.}
\label{Nile}
\end{center}
\vskip -0.2in
\end{figure}

\textbf{Model selection:}
%WILL REWRITE ONCE SPATIOTEMPORAL MODEL POSTERIOR PLOTS NICE ENOUGH
In most of the experiments where abrupt changes model the non-stationarity well, the model posterior is fairly concentrated and periods of model uncertainty 
%as during time periods $100-200$ in panels $3-5$ of \ref{picture_demo} 
are short. This is different when changes are slower, see Fig. \ref{Temperatures}.
%Moreover, Only some of the models in $\mathcal{M}$ ever have a posterior probabily significantly $>0$, suggesting potential on-line speed-up strategies by pruning $\mathcal{M}$. %However, we did not explore this possibility. 
%
%
%Occam's Razor principle 
The implicit model complexity penalization Bayesian model selection performs provides \BOCPDMS with an Occam's Razor mechanism:
%Generally, since Bayesian inference implicitly penalizes model complexity,
Simple models are typically favoured until evidence for more complex dynamics accumulates.
For the bee waggle and the $30$ Portfolio data set, %the algorithm prefers 
%models allowing %the enhanced inferential strength by 
%series to borrow inferential strength from one another via $\{A_l^L\}$ are preferred. 
\BVARs are preferred to \BARs. For the $30$ Portfolio data, the \MAP segmentation only selects \SSBVARs with neighbourhoods %systems
constructed from contemporaneous correlation and autocorrelations. % SSBVARs constructed from the 
Neighbourhoods using \SIC codes are not selected, reflecting that this classification from $1937$ is out of date.
% the $30$ portfolio data, the algorithm strictly sticks to SSBVAR models and does not select any of the BAR-models. In particular, SSBVARs with neighbourhood systems constructed from contemporaneous as well as autocorrelation are selected. The neighbourhoods constructed from the official SIC codes are never selected, reflecting that this classification which has been used since $1937$ is out of date.
\begin{figure}[t!]
\vskip 0.05in
\begin{center}
\centerline{\includegraphics[trim={0.8cm 0.35cm 2.0cm 1.5cm},clip, width=1.0\columnwidth]{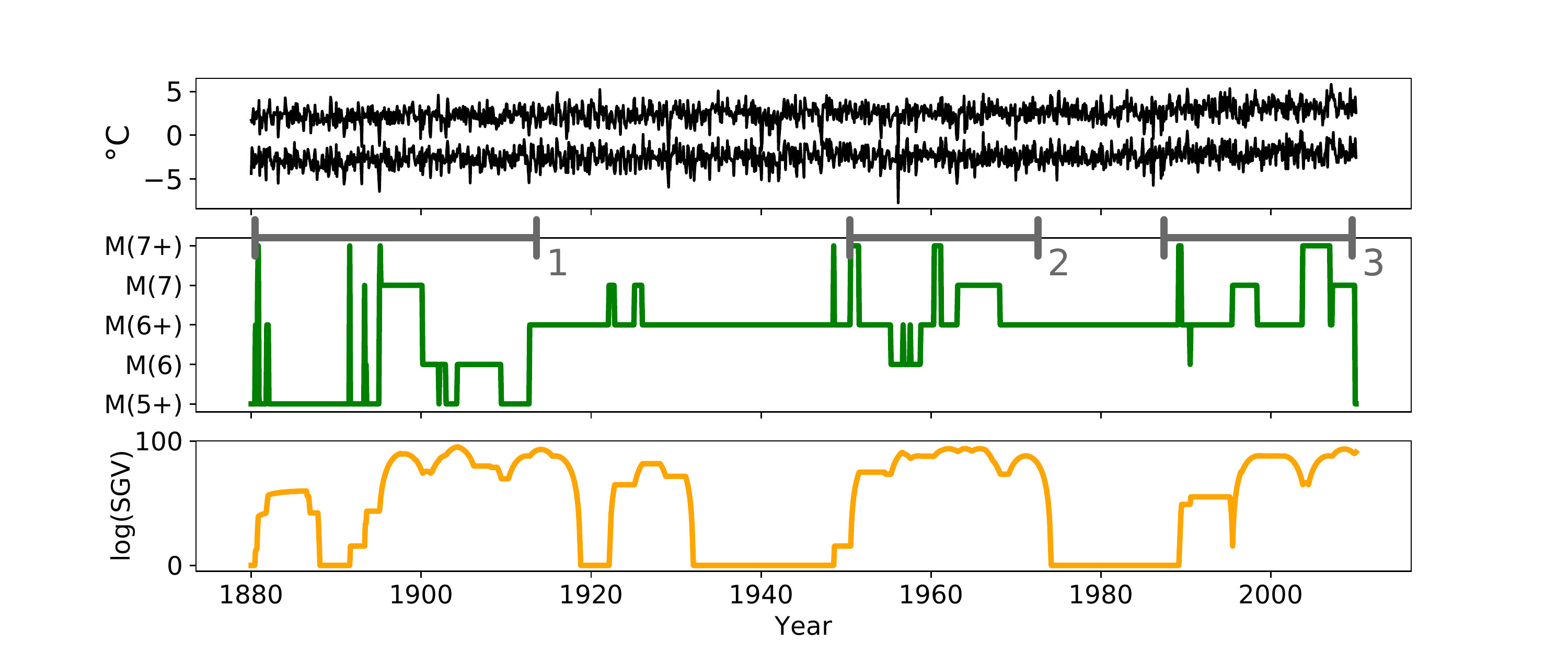}}
%left lower right upper
%DO: Check if we can make textbf a little less fat
\caption{\textit{Results for European Temperatures}: {\color{darkgray}\textbf{Panel 1:}}  normalized temperature for Prague and Jena %(Prague, Jena, Verstervig, Kremsmuenster). 
{\color{darkgray}\textbf{Panel 2:}} {\color{darkgreen} Model Posterior maximum, ${\widehat{m}_t = {\arg\max_{m_t\in\mathcal{M}}\{p(m_t|y_{1:t})\}}}$}, model complexity decreasing top to bottom. $M(l), M(l+)$ are \SSBVAR with $l$ lags. Spatial dependence in $M(l+)$ is slower decaying. Periods of model uncertainty are (1) $2$nd Industrial Revolution $1870-1914$, (2) Post WW2 boom $1950-1973$, (3) European Climate shift $1987-$present, see \citet{EUTemps}. {\color{darkgray}\textbf{Panel 3:}} To compare model uncertainty across different data and $\mathcal{M}$, the (Log) {\color{orange} Standardized Generalized Variance (\SGV)} of $\widehat{m}_t$ can be used. %computed as determinant of the covariance matrix of ${\arg\max_{\mathcal{M}}\{\mathbb{P}(m_t|y_{1:t})\}}$.
}
\label{Temperatures}
\end{center}
\vskip -0.2in
\end{figure}

\subsection{Performance on spatio-temporal data}

%Next, we illustrate how BOCPDMS can be used for inference in real world spatio-temporal and non-stationary data. 
%$21$ European temperature series from $1880-2010$ using monthly averages and $17$ pollution sensors across London measuring nitrogen oxides (NOX) levels.

\textbf{European Temperature:}
Monthly temperature averages $01/01/1880-01/01/2010$ for the $21$ longest-running stations across Europe are taken from {http://www.ecad.eu/}. We adjust for seasonality by subtracting monthly averages for each station.
%We deseasonalize using monthly dummies for each station.
%
Station longitudes and latitudes are available, so $N(\mathcal{S})$ is based on  concentric rings around the stations using Euclidean distances. 
%The set of admissible lag lengths chosen via Thm. \ref{VAR_Thm}, part (3) is  $\{1,2,\dots,7\}$. 
Two different decay functions $\Pi(\cdot), \Pi^{+}(\cdot)$ are used, with $\Pi^{+}(\cdot)$ using larger neighbourhoods and slower decaying.
Temperature changes are poorly modeled by \CPs and more likely to undergo slow transitions. Fig. \ref{Temperatures} shows the way in which the model posterior %$\widehat{m}_t = \arg\max_{M \in \mathcal{M}}\mathbb{P}(m_t|\*y_{1:t})$ 
captures such longer periods of change in dynamics. 
The values on the bottom panel are calculated by considering $\widehat{m}_t = \arg\max_{m_t \in \mathcal{M}}p(m_t|\*y_{1:t})$ as $|\mathcal{M}|$-dimensional multinomial random variable. Its Standardized Generalized Variance (\SGV) \cite{GV, SGV} is calculated as $|\mathcal{M}|$-th root  of the covariance matrix determinant. We plot the log of the \SGV computed using the model posteriors for the last $8$ years. This provides an informative summary of the model posterior dispersion.
%$\arg\max_{M \in \mathcal{M}}\mathbb{P}(m_t|\*y_{1:t})$ as an $|\mathcal{M}|$-dimensional multinomial random variable. Pairwise covariances are computed over a sliding window of $5$ years, and 

%DESCRIBE SGV, GV
%Expect slow changes, look at graph to see that, is indeed what happens, mention data transformations (de-seasonalizing + normalizing)

\textbf{Air Pollution:}
Finally, we analyze Nitrogen Oxide (\NOX) observed at $29$ locations across London $17/08/2002 - 17/08/2003$. The quarterhourly measurements are averaged over $24$ hours. Weekly seasonality is accounted for by subtracting week-day averages for each station.
$\mathcal{M}$ is populated with \SSBVAR models whose neighbourhoods are constructed from both road- and Euclidean distances. 
%BOCPDMS uniformly prefers models built using Euclidean distances, indicating that the pollutant dispersion occurs directly through the air rather than through the road network. 
As $17/02/2003$ marks the introduction of London's first ever congestion charge, we find structural changes in the dynamics around that date. A model with shorter lag length but identical neighbourhood structure is preferred after the congestion charge. In Fig. \ref{AirPollution}, 
Bayes Factors (\BFs) capture the shift: \citet{BF} classify logs of \BFs  as very strong evidence if their absolute value exceeds $5$.

%We expect the dynamic behaviour to change shortly after $17/02/2003$, as this date marks the introduction of London's congestion charge.

%REDO PLOT BY CHANGING COLORS DEPENDING ON WHETHER OR NOT LOG>5

\section{Conclusion}\label{Discussion}

We have extended Bayesian On-line Changepoint Detection (\BOCPD) to multiple models  %-  - and to inference in multivariate, spatio-temporal data 
by generalizing \citet{FearnheadOnlineBCD} and \citet{BOCD}, arriving at \BOCPDMS. 
%obtaining the results of \citet{FearnheadOnlineBCD} using the recursions of \citet{BOCD}. 
%
For inference in multivariate data streams, we propose  \BVARs with closed form distributions that have strong theoretical guarantees summarized in Thm. \ref{VAR_Thm}. 
We sparsify \BVARs based on neighbourhood systems, thus making \BOCPDMS especially amenable to spatio-temporal inference.
To demonstrate the power of the resulting framework, we apply it to multivariate real world data, outperforming the state of the art. %
%
%In future work, we would like to give principled on-line ways of adding and removing models from $\mathcal{M}$.  
In future work, we would like to 
add and remove models from $\mathcal{M}$ on-line. 
This could lower the computational cost for the case where $|\mathcal{M}|$ is significantly larger than the number of threads.

\begin{figure}[h!]
\vskip 0.1in
\begin{center}
\centerline{\includegraphics[trim={0.5cm 0.8cm 1.0cm 1.5cm},clip, width=1.0\columnwidth]{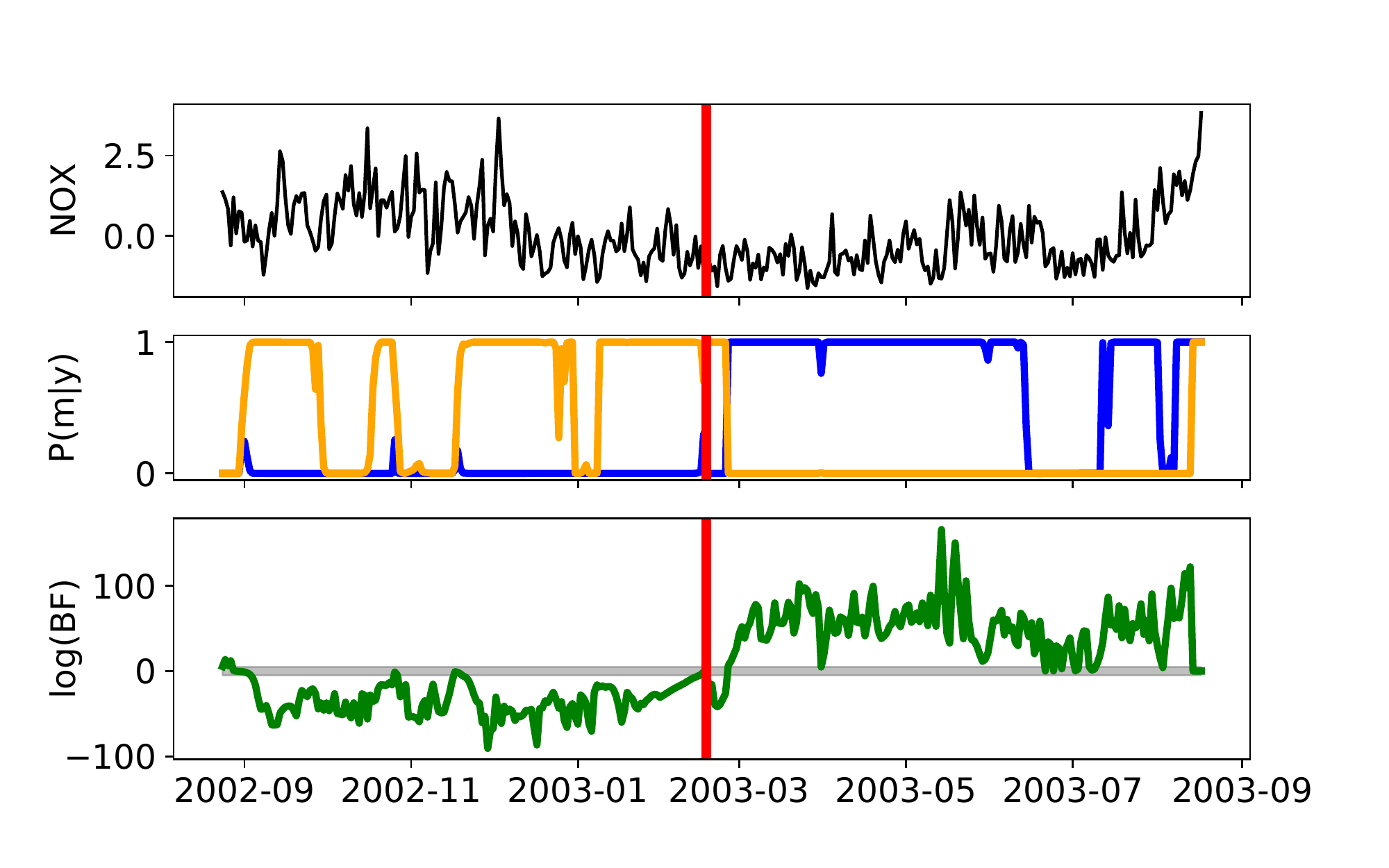}}
%left lower right upper
%DO: Check if we can make textbf a little less fat
\caption{\textit{Results for Air Pollution}: {\color{darkgray}\textbf{Panel 1:}}  {\textbf{\NOX}} levels for Brent, with {\color{red}congestion charge introduction} date %(Prague, Jena, Verstervig, Kremsmuenster). 
{\color{darkgray}\textbf{Panel 2:}}  Model posteriors for the two best-fitting models, with Euclidean neighbourhoods. {\color{darkgray}\textbf{Panel 3:}} Their log Bayes Factors, $[-5,5]$ \textcolor{gray}{\textbf{shaded}}. %computed as determinant of the covariance matrix of ${\arg\max_{\mathcal{M}}\{\mathbb{P}(m_t|y_{1:t})\}}$.
}
\label{AirPollution}
\end{center}
\vskip -0.2in
\end{figure}

%
%It would also be desirable to make BOCPDMS fully Bayesian as in \citet{HazardLearningBOCD}, thus avoiding hyperparameter optimization.
%
%Future work/limitations: pre-specifying model universe, priors over hazard function (i.e. make it more bayesian instead of performing hyperparameter-opt), 
%prune M.
% for multivariate data. 
%The performance of the algorithm on real world data demonstrates the implications, and outperforms the state of the art.
% demonstrated on real world data. 
%type-II ML for hyperparams + mention other work that has been done

%\newpage
\section*{Acknowledgements}

%\textcolor{blue}{Capitalizing references}

We want to thank N. Karampatziakis for his help with making the method computationally more efficient.
JK is funded by \EPSRC grant EP/L016710/1. 
Further, this work was supported by The Alan Turing Institute for Data Science and AI under \EPSRC grant EP/N510129/1 and the Lloyds Register Foundation programme on Data Centric Engineering.

\bibliography{library}
\bibliographystyle{icml2018}

%\newpage

\appendix
\addcontentsline{toc}{section}{Appendices}

\section{Condition of Theorem 1}

\textit{Denoting the spectrum of a matrix $\*B$ (i.e., the set of its eigenvalues) by $\sigma(\*B)$, the following condition is a restatement of the relevant part in condition \textbf{A} of Meyer \& Kreiss (2015)}:

\begin{condition}
 Let $\*W$ be the spectral density matrix of the purely non-deterministic stochastic process $\{\*Y_t\}_{t=1}^{\infty}$ satisfying the conditions of Theorem 1.  We assume that the spectral density matrix is bounded, i.e.\ there is a constant $c>0$ so that
 \begin{IEEEeqnarray}{rCl}
 	\min\left( \sigma(\*W(\lambda))  \right) & \geq & c
 \end{IEEEeqnarray}
 for all frequencies $\lambda \in (-\pi, \pi]$, i.e.\ the eigenvalues of the spectral density matrix are uniformly bounded
away from zero.
\end{condition}

\section{Empirical evaluation of computation time}

For this comparison, we use the original code of Turner (2012) for the \GP-models. As the \MSE is smallest for \ARGPCP for all data sets except for the snowfall data, we compare \BOCPDMS against the arguably best \GP \CP model. We note that while \NSGP performs better on the snowfall data than \ARGPCP, its requirement to do Hamiltonian Monte Carlo sampling will make it significantly slower. 
We also note that \BVAR models inside \BOCPDMS outperformed the \MSE of the \ARGPCP model for all data sets considered. All computations were performed on a  3.1 \GHz Intel \iseven{ }with 16\GB \RAM. 

Table \ref{benchmark_table} summarizes the results. It is clear that \BOCPDMS outperforms \ARGPCP computationally: e.g., the computation time per parameter is between $60$ (Nile data) and $585$ (Bee data) times faster for \BOCPDMS with \BVAR models. Computation times are faster per model, too. The only exception to this is the $30$ Portfolio data set, where the deployed \SSBVAR models are orders of magnitude more parameter-rich than the \ARGPCP-model.
Related to this, we also note that comparing the computation time per parameter makes sense for two reasons: Firstly, \BVARs model the $d$ time series jointly, thus requiring $d^2$ parameters in the posterior covariance matrix of $\*y_t$. In contrast, the \GP-models  ignore any dependence between the series, resulting in $d$ parameters of the (diagonal) posterior covariance matrix for $\*y$. Secondly, the parameters of the \GP's kernel arguably making its parameter space $\Theta$ infinite-dimensional are not actually learnt on-line at all. Instead, they are optimized for a training period of $T'$ observations and then fixed, see  section 4 in  the main paper. Hence, the parameter space the \GP-models can learn in is  finite-dimensional.

%
%We also note that the kernel hyperparameters of the \GP determining its non-parametric nature  are not learnt at all (after the initial training period of length $T'$) except for the \NSGP model. This also means that 

%\textcolor{blue}{Explain why we should compare on parameter-basis}

%Table of (1) total time needed, (2) parameters that we do Bayesian inference for in GP & BOCPDMS, (3) time per model, (4) time per parameter

%make remarks that we have  a constant C that grows quadratically in the number of regressors, so that is why high-dim data is slower

%Maybe present graph? (computation time per param + per model)

\begin{table}[h!]
{\renewcommand{\arraystretch}{1.2}
\caption{Computation time in seconds per model and per parameter in the space $\Theta = \cup_{m \in \mathcal{M}}\Theta_m$}
% \NLL marked $^{\ast}$ are centered in  \citet{TurnerThesis}  so that the \NLL of the best-performing \GP-method amongst them is $0.0$.}
%Use some dots or hearts to separate table entries
\label{benchmark_table}
%\vskip 0.1in
\begin{center}
\vskip 0.1in
\begin{small}
\begin{sc}
\begin{tabular}{ p{1.2cm}  p{2.2cm} p{2,2cm}}
 %\hline\hline
  \multicolumn{1}{c}{}& \multicolumn{2}{c}{Nile}\\
\hline
       	& Time/$|\mathcal{M}|$ & Time/$|\Theta|$ \\[1.5pt]
 \hline
 \ARGPCP    	& $42.2$		&  $21.0$   	\\
 \BVAR		  	&	$\mathbf{4.03}$		& $\mathbf{0.35}$    		\\
 \hline\hline
  \multicolumn{1}{c}{}& \multicolumn{2}{c}{Snowfall} \\
\hline
 & Time/$|\mathcal{M}|$ & Time/$|\Theta|$ \\[1.5pt]
 \hline
 \ARGPCP   	 	& $284$		&  $142$   \\
 \BVAR		 & $\mathbf{157}$	& $\mathbf{4.25}$  
 			\\
 \hline\hline
   \multicolumn{1}{c}{}& \multicolumn{2}{c}{Bee}\\
\hline
      	& Time/$|\mathcal{M}|$ & Time/$|\Theta|$ \\[1.5pt]
 \hline
 \ARGPCP   		& $164$		&  $23.4$   	\\
 \BVAR		 	&	$\mathbf{97.3}$		& $\mathbf{0.04}$    		\\
 \hline\hline
  \multicolumn{1}{c}{}& \multicolumn{2}{c}{$30$ Portfolios} \\
\hline
  	& Time/$|\mathcal{M}|$ & Time/$|\Theta|$ \\[1.5pt]
 \hline
 \ARGPCP  	& 	$\mathbf{12077}$	&  $403$   \\
 \BVAR	 &  ${34183}$	& $\mathbf{1.48}$  \\ 
% \hline\hline
\end{tabular}
\end{sc}
\end{small}
\end{center}}
\end{table}

\end{document}